\documentclass[sigconf]{acmart}
\AtBeginDocument{%
  }

\copyrightyear{2026}
\acmYear{2026}
\setcopyright{cc}
\setcctype{by}
\acmConference[KDD 2026] {Proceedings of the 32nd ACM SIGKDD Conference on Knowledge Discovery and Data Mining V.2}{August 9--13, 2026}{Jeju Island, Republic of Korea.}
\acmBooktitle{Proceedings of the 32nd ACM SIGKDD Conference on Knowledge Discovery and Data Mining V.2 (KDD 2026), August 9--13, 2026, Jeju Island, Republic of Korea}
\acmISBN{979-8-4007-2259-2/2026/08}
\acmDOI{10.1145/3770855.3818431}

\usepackage{hyperref}
\usepackage{url}

\usepackage{algorithm}
\usepackage[noend]{algorithmic}

\usepackage[utf8]{inputenc}
\usepackage[T1]{fontenc}
\usepackage{booktabs}       % professional-quality tables
\usepackage{amsfonts}       % blackboard math symbols
\usepackage{nicefrac}       % compact symbols for 1/2, etc.
\usepackage{microtype}      % microtypography
\usepackage{xcolor}         % colors
\usepackage{subfigure}      % for subfigures (on ACM accepted-packages list)

\usepackage{amsmath}
\usepackage{mathtools}
\usepackage{amsthm}
\usepackage[capitalize,noabbrev]{cleveref}

\usepackage{svg}
\usepackage[pdf]{graphviz}
\usepackage{xpatch}

\theoremstyle{plain}
\newtheorem{theorem}{Theorem}[section]

\newtheorem{claim}[theorem]{Claim}
\newtheorem{corollary}[theorem]{Corollary}
\theoremstyle{definition}
\newtheorem{definition}[theorem]{Definition}

\theoremstyle{remark}

\newtheorem{property}{Property}

\makeatletter
\newcommand*{\addFileDependency}[1]{% argument=file name and extension
  \typeout{(#1)}
  \@addtofilelist{#1}
  \IfFileExists{#1}{}{\typeout{No file #1.}}
}
\xpretocmd{\digraph}{\addFileDependency{#2.dot}}{}{}
\makeatother

\renewcommand{\algorithmicrequire}{\textbf{Input:}}
\renewcommand{\algorithmicensure}{\textbf{Output:}}

\def\cD{\mathcal{D}}
\def\cC{\mathcal{C}}
\def\cP{\mathcal{P}}
\def\cO{\mathcal{O}}
\def\exp{\mathbb{E}}

\def\eqdef{\stackrel{\text{def}}{=}}
\definecolor{darkgreen}{rgb}{0.0, 0.5, 0.0}
\def\EXP32{\ifmmode \mathrm{ABoB} \else ABoB\fi}
\newcommand{\card}[1]{\lvert #1 \rvert}
 
\newcommand{\roey}[1]{#1}
\newcommand{\revision}[1]{#1}
\newcommand{\revisiontwo}[1]{#1}
\newcommand{\revisionKDD}[1]{{#1}}
\newcommand{\revisionCR}[1]{{#1}}
\newcommand{\para}[1]{\paragraph{\bf #1}}

\begin{document}

\title{Hierarchical Adversarial Bandits for Online Configuration Optimization}

\author{Gil Shabat}
\authornote{These authors contributed equally to this research.}
\email{gshabat@nvidia.com}
\orcid{0000-0002-7798-3194}
\affiliation{%
  \institution{NVIDIA}
  \city{Tel Aviv}
  \country{Israel}}

\author{Chen Avin}
\authornotemark[1]
\email{avin@bgu.ac.il}
\orcid{0000-0002-6647-8002}
\affiliation{%
  %\department{School of Electrical and Computer Engineering}
  \institution{Ben-Gurion University of the Negev}
  \city{Beer Sheva}
  \country{Israel}}
  \affiliation{%
  \institution{NVIDIA}
  \city{Tel Aviv}
  \country{Israel}}

\author{Shie Mannor}
\email{smannor@nvidia.com}
\orcid{0000-0003-4439-7647}
\affiliation{%
  \institution{NVIDIA}
  \city{Tel Aviv}
  \country{Israel}}

\author{Hanan Shteingart}
\authornotemark[1]
\email{hshteingart@nvidia.com}
\orcid{0000-0002-7117-4337}
\affiliation{%
  \institution{NVIDIA}
  \city{Tel Aviv}
  \country{Israel}}

\author{Zvi Lotker}
\email{lotkerz@biu.ac.il}
\orcid{0000-0002-3759-5584}
\affiliation{%
  \institution{Bar-Ilan University}
  \city{Ramat Gan}
  \country{Israel}}

\author{Roey Yadgar}
\authornotemark[1]
\authornote{Work done while at NVIDIA.}
\email{roaiyadgar@gmail.com} 
\orcid{0009-0008-0177-8160}
\affiliation{
  \institution{NVIDIA}
  \city{Tel Aviv}
  \country{Israel}}

\renewcommand{\shortauthors}{Gil Shabat et al.}

\begin{abstract}
\revisionKDD{Motivated by Online Configuration Optimization in large, dynamic parameter spaces, this work studies the nonstochastic multi-armed bandit (MAB) problem in metric action spaces with oblivious Lipschitz adversaries.
We propose ABoB (Adversarial Bandit over Bandits), a hierarchical framework that decomposes the configuration space into clusters to accelerate learning and adapt to changing environments.
We evaluate ABoB using standard algorithms such as EXP3 and Tsallis-INF on a real-world production storage system, demonstrating significant performance gains of up to $50\%$ compared to state-of-the-art "flat" bandit algorithms. 
Extensive simulations further confirm that ABoB effectively exploits metric structures, achieving up to $91\%$ improvement in adversarial metric scenarios while significantly reducing computational running time. Theoretical analysis grounds this empirical success: we prove that ABoB maintains a worst-case "safety net" bound of $O(\sqrt{kT})$—matching traditional methods, where $T$ is the number of rounds and $k$ is the number of arms, while capable of accelerating learning to $O(k^{1/4}\sqrt{T})$ under favorable Lipschitz conditions. This combination of operational efficiency and theoretical soundness makes ABoB a practical solution for automated system tuning.}
%
% We prove that in the worst-case scenario, such a clustering approach cannot hurt too much and \EXP32 guarantees a standard worst-case regret bound of $\cO(k^{\frac{1}{2}}T^{\frac{1}{2}})$, where $T$ is the number of rounds and $k$ is the number of arms, matching the traditional flat approach.  
% However, under favorable conditions related to the algorithm properties, clusters properties, and certain Lipschitz conditions, the regret bound can be improved to $\cO(k^{\frac{1}{4}}T^{\frac{1}{2}})$. 
% Simulations and experiments on a real storage system demonstrate that \EXP32, can be made practical using standard algorithms like EXP3 and Tsallis-INF. \EXP32 achieves lower regret and faster convergence than the flat method, up to 50\% improvement in known previous setups, nonstochastic and stochastic, as well as in our settings.
\end{abstract}

\begin{CCSXML}
<ccs2012>
<concept>
<concept_id>10003752.10010070.10010071.10010261.10010272</concept_id>
<concept_desc>Theory of computation~Sequential decision making</concept_desc>
<concept_significance>500</concept_significance>
</concept>
<concept>
<concept_id>10003752.10003809.10010047</concept_id>
<concept_desc>Theory of computation~Online algorithms</concept_desc>
<concept_significance>500</concept_significance>
</concept>
<concept>
<concept_id>10010147.10010257.10010258.10010261.10010272</concept_id>
<concept_desc>Computing methodologies~Sequential decision making</concept_desc>
<concept_significance>500</concept_significance>
</concept>
<concept>
<concept_id>10010147.10010257.10010282.10010284</concept_id>
<concept_desc>Computing methodologies~Online learning settings</concept_desc>
<concept_significance>300</concept_significance>
</concept>
<concept>
<concept_id>10002951.10003152.10003520</concept_id>
<concept_desc>Information systems~Storage management</concept_desc>
<concept_significance>100</concept_significance>
</concept>
</ccs2012>
\end{CCSXML}

\ccsdesc[500]{Theory of computation~Sequential decision making}
\ccsdesc[500]{Theory of computation~Online algorithms}
\ccsdesc[500]{Computing methodologies~Sequential decision making}
\ccsdesc[300]{Computing methodologies~Online learning settings}
\ccsdesc[100]{Information systems~Storage management}

\keywords{Adversarial multi-armed bandit, Hierarchical bandits, Configuration optimization, 
Lipschitz condition, Online algorithms}


\maketitle

\section{Introduction}
% The multi-armed bandit (MAB) problem is a fundamental framework in decision theory and machine learning, elegantly capturing the exploration-exploitation dilemma inherent in many real-world scenarios, from clinical trials and online advertising to resource allocation and dynamic pricing. In its simplest form, an agent repeatedly needs to choose from a set of actions (``arms''), each yielding a reward, with the goal of maximizing cumulative reward over time. The challenge lies in balancing the need to explore different arms to learn their reward distributions while exploiting the arms that have so far yielded the highest rewards \cite{slivkins2019introduction,bergemann2006bandit,bubeck2012regret}.

The multi-armed bandit (MAB) problem is a fundamental concept in decision theory and machine learning. It effectively illustrates the exploration-exploitation dilemma that arises in many real-world situations, such as clinical trials, online advertising, resource allocation, and dynamic pricing \cite{villar2015multi,schwartz2017customer}. 
In its simplest form, an agent must repeatedly choose from a set of actions, known as ``arms,'' each of which provides a reward. The objective is to maximize the cumulative reward over time, or equally, minimizing the \emph{regret}. The primary challenge is finding the right balance between exploring different arms to understand their reward distributions and exploiting those arms that have previously generated the highest rewards \cite{slivkins2019introduction,bergemann2006bandit,bubeck2012regret}.
%
% The MAB framework has seen significant theoretical advancements and practical applications, providing valuable insights into optimal decision-making under uncertainty \cite{djallel2019survey}. However, many real-world applications deviate from the classical MAB setting's assumptions, particularly when dealing with evolving environments and complex relationships between choices \cite{auer2002nonstochastic,slivkins2008adapting}.  
%
%The MAB framework has undergone substantial theoretical advancements 
The MAB framework has made significant theoretical progress over the years, and its practical applications offer valuable insights into optimal decision-making under uncertainty \cite{djallel2019survey}. However, many real-world applications diverge from the assumptions of the classical MAB setting, especially when addressing evolving environments and complex relationships among choices \cite{auer2002nonstochastic, slivkins2008adapting}.

% This paper addresses one such departure, introducing a novel framework for the multi-armed bandit problem in a dynamic, nonstochastic (adversarial) environment. The setup is inspired by real-world systems involving a very large number of optimization parameters, for example, in automated configuration tuning for computing/storage systems. Automatic configuration and tuning in real-world settings, ranging from industrial machines to smart home appliances, has become a critical challenge due to the increasing complexity and diversity of modern systems and devices. In the world of datacenters and high-performance computing (HPC), one can find GPU kernel optimization using Bayesian optimization  \cite{willemsen2021bayesian}, online energy optimization in GPUs \cite{xu2024online}, or hyperparameter optimization \cite{li2018hyperband}, to name a few. Specifically, we applied this method to optimize a real storage system with a large configuration space. To deal with ample search space, we consider scenarios where arms are set in a metric space and partitioned into clusters that exhibit Lipschitz properties. Unlike traditional bandit settings, we focus on the nonstochastic landscape where the reward of each arm can evolve over time, reflecting, for example, the changing performance of different system configurations under a fluctuating workload, like a process of jobs arriving and finishing. Crucially, the arms are {\em not }independent; they are grouped into clusters based on a predefined metric representing the similarity of configurations.  

This paper addresses one such departure, studying a framework for the multi-armed bandit problem in a dynamic, nonstochastic (adversarial) environment. The setup is inspired by real-world systems that involve a vast number of optimization parameters, such as automated configuration tuning for computing and storage systems. Automatic configuration and tuning in various real-world settings—ranging from industrial machines to smart home appliances—have become critical challenges due to the increasing complexity and diversity of modern systems and devices.
In the context of data centers and high-performance computing (HPC), we can find examples of this, such as GPU kernel optimization using Bayesian optimization \cite{willemsen2021bayesian}, online energy optimization in GPUs \cite{xu2024online}, and hyperparameter optimization \cite{li2018hyperband}. 
\revision{The method and work we report here were specifically applied to design an automated method that optimizes a real storage system employed in a large distributed computation cluster that has a large configuration space.}

\revisiontwo{In our scenario, each system configuration (a specific setup of parameters) is considered as an arm.}
%We \revision{mostly } consider scenarios where the
Configurations are situated in a metric space and partitioned into clusters that exhibit Lipschitz properties, which allow us to handle the extensive search space. Unlike traditional bandit settings, our focus is on a nonstochastic landscape where the reward for each arm can change adversarially over time, \revision{including the optimal arm (configuration). } This reflects, for example, the variation in performance of different system configurations under fluctuating workloads, such as the dynamic arrival and completion of job/tasks in the system. Importantly, the arms are not independent; they are grouped into clusters based on a predefined metric that represents the similarity of configurations. The arms' clustering can either be given as input or initially computed by the MAB algorithm. 

In each cluster of arms, the rewards adhere to a Lipschitz condition. This means that configurations that are ``close'' (i.e., have similar parameter settings) will yield similar performance results, even as the performance shifts with the environment dynamics. This conveys that while the \emph{optimal configuration} for a system may change with different workloads, configurations with similar parameter settings are likely to show comparable performance levels. The "traveling arms" framework requires algorithms that efficiently learn and adapt to an evolving reward landscape while taking advantage of the structural information offered by the partition and the Lipschitz characteristics of the arms. We explore the challenges and opportunities inherent in this new bandit setting, presenting algorithms and theoretical guarantees aimed at minimizing regret in this dynamic and practical environment.

\begin{figure*}[t]
    \centering
    \subfigure[System overview]{
    \label{fig:overview}
    \includegraphics[width=.65\linewidth]{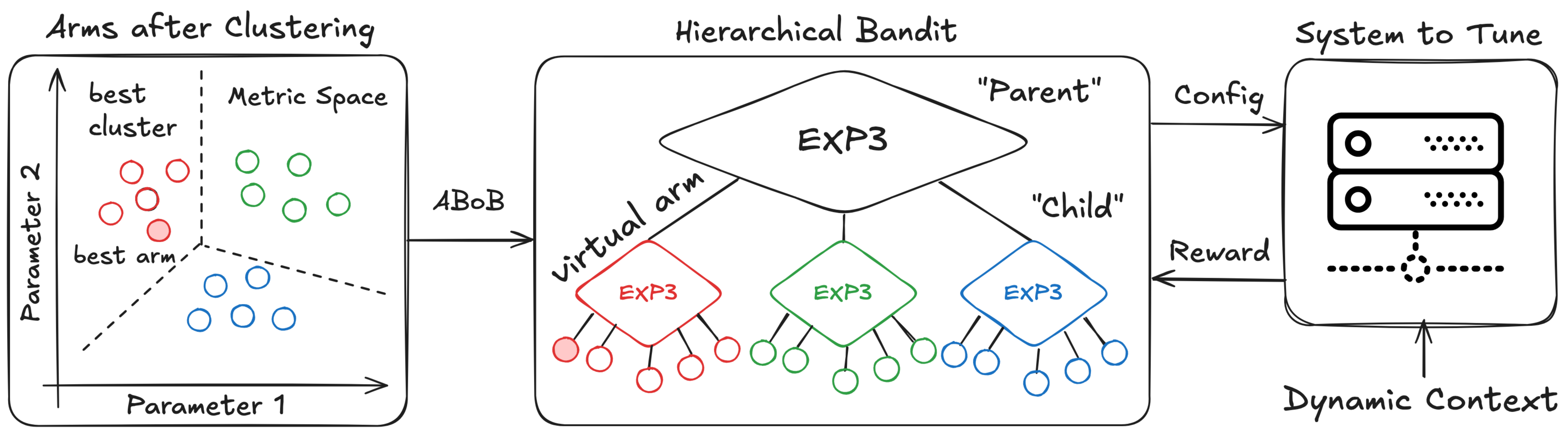}}
    \subfigure[Results' summary]{
    \label{fig:table}
    \includegraphics[width=.3\linewidth]{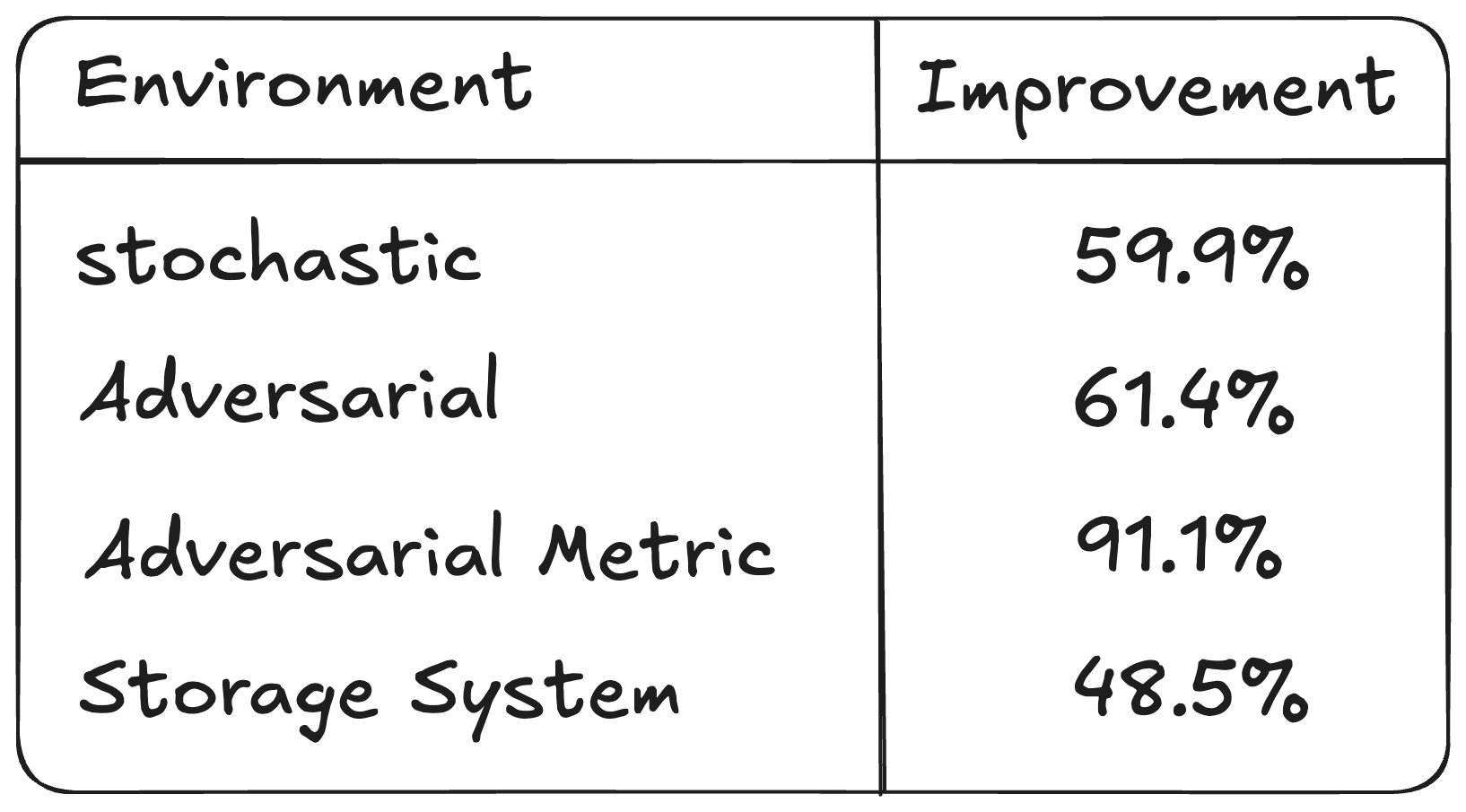}}
    \caption{
    (a) System overview of Adversarial Bandit over Bandits (\EXP32).
    (left) Partition the configurations (arms) into clusters of similar configurations.
    (middle) Hierarchical bandits over the clusters (\emph{virtual arms}).
    (right) Optimize a system's performance under a dynamic context.
    (b) A summary of \EXP32's results.
    %\revision{A high-level description of Adversarial Bandit over Bandits (\EXP32).
    % (a) To optimize a system's performance under dynamic context (right), we partition the parameters' configurations (arms) into clusters of close configurations. In turn, we continuously run an adversarial MAB algorithm (e.g., EXP3) to select between clusters (\emph{virtual arms}) and then, using a unique adversarial MAB (e.g., EXP3) for each cluster, select the arm within the cluster (middle). We set this configuration in the system, and upon receiving its reward, we update the weights of the arm and the cluster. (b) A summary of \EXP32's results.
    %}
    }
    \Description{System-overview diagram and a small summary table illustrating ABoB.}
    \label{fig:system}
\end{figure*}

To leverage the structural information inherent in the problem setup, we propose a hierarchical MAB algorithm called \revision{\emph{Adversarial Bandit over Bandits} (\EXP32). } Refer to Figure \ref{fig:system} for an illustration.
Our algorithm utilizes multiple instances of other well-known adversarial bandit algorithms \revision{like the classical EXP3 \cite{auer2002nonstochastic} or Tsallis-INF \cite{zimmert2021tsallis}}, known to have the ``best of two worlds" property. 
Given the arms' partition, at the \revision{first level (tier) of our hierarchy, a \emph{parent} bandits algorithm (e.g., EXP3 or Tsallis-INF),  looking at each cluster as a \emph{virtual arm}, } determines which cluster to engage in the current time step. Subsequently, a second-level \revision{algorithm, denoted as a \emph{child} bandit algorithm (e.g., EXP3 or Tsallis-INF)}, specific to each cluster, is activated to select the next arm within the chosen cluster. 
%\gil{We talk too much on EXP3. Do we need to replace it by a general bandit setting? I think we need to mention that the parent level must be a bandit with bounds for nonstochastic case.}
%\gil{You mention in the introduction that we *mostly* discuss the case were the arms are clustered, but say nothing on the case they are not, eventhough we discuss it later on in the paper. Perhaps this also should at least be mentioned here as well.}
By design, employing known adversarial algorithms allows us to benefit from their known strengths in adversarial settings. At the same time, the clustering approach enables us to exploit the underlying metric, as similar configurations tend to yield similar outcomes (i.e., they satisfy the Lipschitz condition). 
\revisionKDD{While hierarchical 'bandit-over-bandit' structures have been explored for model selection \cite{agarwal2017corralling}, our work uniquely adapts this architecture for spatial decomposition in metric spaces. This allows us to exploit local Lipschitz properties in configuration tuning without the complexity of meta-learning regularizers.}

\para{Paper Contributions.} 
This paper presents several key contributions:
(i) \emph{A Lipschitz Clustering-based MAB Variant}: We introduce a new variant of the Multi-Armed Bandit (MAB) problem that addresses the challenges of dynamic environments with large, structured action spaces, particularly in the context of system configuration optimization.
(ii) \emph{Hierarchical Clustering and the \EXP32 Algorithm}:
\revisionKDD{We propose Adversarial Bandit over Bandits (\EXP32), a hierarchical algorithm designed for large configuration spaces. \EXP32 clusters configurations into "virtual arms" to exploit local structures, using standard MAB algorithms (like EXP3 or Tsallis-INF) at both the cluster and arm levels to adapt to changing environments.}
(iii) \emph{Real-World Impact \& Efficiency}: \revisionKDD{We validate \EXP32 on a real production storage system, where it achieved approximately $50\%$ lower regret compared to state-of-the-art flat methods. Extensive simulations further demonstrate improvements of up to $91\%$ in metric spaces and confirm that \EXP32 significantly reduces computational running time compared to flat baselines.} See Figure \ref{fig:table} for a summary of the results.
(iv) \emph{Theoretical Safety Net}: \revisionKDD{We provide a rigorous analysis proving that \EXP32 is safe to deploy: in the worst case, its regret bound is $O(\sqrt{kT})$, matching the standard "flat" approach, where $T$ represents the number of iterations and $k$ denotes the number of arms \revision{that is finite but assumed to be large}. This performance matches the bound of the "flat"  method (e.g., of Tsallis-INF \cite{zimmert2021tsallis}). However, under favorable Lipschitz conditions, we prove the bound improves to $O(k^{1/4}\sqrt{T})$. This ensures that using \EXP32 offers significant potential gains without risking worse performance than traditional methods.}
Due to space constraints, additional experiments and details are present in the technical appendix.

% \emph{Disclaimer}. %For clarification, 
% It is important to note that the paper does not claim that hierarchical clustering or the approach of \EXP32 are optimal in theory or best in practice. The goal is to study a flat EXP3 versus a hierarchical approach.   

\section{Related Work}
% Finite stochastic case
% ---------
The multi-armed bandit (MAB) problem has been extensively studied under various assumptions, leading to a rich landscape of algorithms and theoretical results \cite{slivkins2019introduction}. Early work focused primarily on the stochastic setting, where the finite set of arm rewards are IID random variables drawn from fixed, unknown distributions. Algorithms like Upper Confidence Bound (UCB) \cite{ucb2002auer} and Successive Elimination \cite{even2002pac} provide strong performance guarantees in this scenario by balancing exploration and exploitation based on estimated reward distributions' means and confidence intervals. 
The canonical worst-case scenario of ``needle in a haystack'' scenario leads to a regret of $\Tilde{\cO}(\sqrt{kT})$ (ignoring polylogarithmic factors) for these algorithms.
%----------

% infinite stochastic case
%----------
Another line of work deals with continuous action spaces, still under stochastic settings. The most straightforward approach is performing uniform discretization on the arms, under assumptions of metric spaces and smoothness assumptions, 
and then applying a known multi-armed bandit approach (like UCB). More advanced techniques include approaches like the \emph{zooming algorithm} \cite{kleinberg2008multi,kleinberg2019bandits}, 
which takes a more adaptive and dynamic approach. It starts with a coarse view of the action space and progressively refines its focus on promising regions. 
It maintains a confidence interval for each arm (or region), and based on these intervals, it either ``zooms in'' on a region by dividing it into subregions or eliminates it if it's deemed suboptimal. 
%----------
% nonstocastic case
A more challenging research direction considers the finite 
nonstochastic setting with a finite number of arms, where adversarial opponents control reward assignments. For this setting to be feasible, the usual assumption is of an oblivious adversary and a \emph{regret} where the algorithm is compared with the best arm in hindsight \revisionKDD{, that is, the best fixed action which would have resulted in the highest cumulative reward}. For this case, algorithms like EXP3 \cite{auer2002nonstochastic} provide robust regret guarantees against any sequence of rewards, surprisingly matching the worst-case bound of the stochastic case.  

% adversarial continuous
More recently, and perhaps the most related approach to the algorithm presented in this paper, the combination of continuous action spaces and adversarial rewards has been explored in the adversarial zooming setting, presenting unique challenges addressed by algorithms that dynamically adapt their exploration strategy based on the observed structure of the reward landscape \cite{podimata2021adaptive}, yielding a regret bound that depends on a new quantity, $z$, called \emph{adversarial zooming dimension} and is given by
%\begin{equation}
$\mathbb{E}[R(T)] \le \Tilde{\cO}( T^{\frac{z+1}{z+2}}).$
%\label{eq:advzoom}
%\end{equation} 
For \emph{finite} number of arms, we have $z=0$, %in Eq. \ref{eq:advzoom} 
and the paper provides a worst-case regret-bound 
%$\mathbb{E}\left[R(T)\right] \le \Tilde{\cO}\left( \sqrt T \right)$ or 
of $\cO(\sqrt{kT \log^5 T})$. This regret bound is similar to that of the nonstochastic multi-armed bandit, i.e. $\Tilde{\cO}(\sqrt{kT})$.
\revisiontwo{For the finite case, the algorithm, therefore, does not offer improved bounds. In contrast, our approach provides conditions for improved performance and demonstrates practical benefits.}
%Our work builds upon algorithms in the nonstochastic finite setting, extending the adversarial setting to incorporate metric space structure and Lipschitz condition within clusters inspired by real-world problems. % like dynamic system configuration. 
%Hierarchical Bandits and Meta-Learning. 
\revisionKDD{Our approach leverages a hierarchical structure that shares similarities with Bandit Model Selection frameworks like CORRAL \cite{agarwal2017corralling,cheung2022hedging}. However, these methods employ a master algorithm to compete against a set of heterogeneous base algorithms (e.g., to select the best learning rate), often requiring complex stability corrections (e.g., log-barrier regularizers) to handle the high variance of importance-weighted feedback. In contrast, ABOB employs hierarchy for spatial decomposition: our "child" bandits solve disjoint sub-problems (clusters) to exploit local geometry rather than competing to solve the global problem.}

\revisionKDD{Similarly, our work relates to Algorithmic Chaining \cite{cesa2017algorithmic}, which uses a deep, multi-scale tree of experts for theoretical minimax optimality in non-parametric learning. While chaining relies on intricate multi-level coverings, ABOB focuses on a practical two-level regime suitable for system configuration. We propose a "flat" partitioning approach that is simpler to implement and allows the direct reuse of off-the-shelf adversarial algorithms (e.g., Tsallis-INF) at both the cluster and arm levels, ensuring ease of deployment in real-world storage systems.}

% Other correlated cases
Other papers studied various settings of correlated or dynamic arms, for example, $\mathcal{X}$–Armed Bandits \cite{bubeck2011x},  Contextual Bandits \cite{slivkins_contextual_2011}, Correlated arms \cite{gupta_multi-armed_2021}
Eluder Dimension \cite{russo2013eluder} and Dependent Arms \cite{pandey2007multi}, to name a few, but none 
\revisionKDD{utilize the parent-child hierarchy to perform dynamic partitioning of the configuration space. Unlike prior meta-learning approaches that compete heterogeneous algorithms against one another, ABOB uses the hierarchy to zoom into promising regions of the metric space using homogeneous base learners.}

%are using our exact setup, nor the novel idea of parent-child bandits. 
%same setup as in this paper.

%We show, that under certain conditions, described in Section \ref{sec:dblexp3}, the upper bound for the regret can be further improved by a factor of $\cO(\sqrt[4]{k})$.
% 1. Stochastic bounds (ucb, zooming
% 2. Continuous arms (Stochastic), uniform, zooming 
% 3. dimension is 0 for the finite case 
% 3. nonstochastic finite (EXP3) 
% 4. (recently) Adversarial zooming only for continuous
% 5. other work on correlated arms but not hierarchical
% 6. we are not aware of similar setting and bounds 

%Explain the limitation of multi-armed bandits where the number of arms is large (or infinite), then review current solutions from the literature, starting from uniform discretization (the simplest), also need to mention zooming algorithm (stochastic), slate bandits, zooming algorithm for nonstochastic and its limitations (upper bounds).

%one line: \cite{kour2014real,kour2014fast,hadash2018estimate,auer2002nonstochastic,slivkins2019introduction,cesa-bianchi_how_1997,bubeck2011x,kleinberg2008multi,gupta_multi-armed_2021,kleinberg2019bandits,koren2017multi,slivkins2008adapting,slivkins_contextual_2011,russo2013eluder,podimata2021adaptive,pandey2007multi}

\section{Problem Formulation and \EXP32 Alg.}
We study the Adversarial Lipschitz MAB (AL-MAB) Problem \cite{podimata2021adaptive}.
In particular, we consider the finite case where the set of arms forms a metric space, and the adversary is oblivious to the algorithm's random choices.
The problem instance is a triple $(K , \cD, \cC)$, where $K$ is the set of $k$ arms $\{1, 2, \dots, k\}$,  $(K , \cD)$ is a metric space, and $\cC$ is an \revision{expected } rewards assignment, i.e., an infinite sequence $\mathbf{c}_1, \mathbf{c}_2,...$ of vectors $\mathbf{c}_t=(c_t(1),...,c_t(k))$, where \revision{$c_t(a)$ is the expected reward of arm $a$ at time $t$, and }  $\mathbf{c}_t  : K \rightarrow [0, 1]$ is 
a Lipschitz function on $(K , \cD)$ \revisionKDD{at every time step $t$. For simplicity, we assume the Lipschitz constant to be 1 (Otherwise the metric $\cD$ can be rescaled)}. Formally,
\begin{equation}\label{eq:Lipschitz}
    \card{c_t(a) - c_t(a')} \le \cD(a, a') 
\end{equation}
for all arms $a, a' \in K$ and all rounds $t$,
where $\cD(a, a')$ is the distance between $a$ and $a'$ in $\cD$.

For our analytical results, we consider a special case where we partition the set of arms into a set of clusters. We let $\cP$ be a partition\footnote{We use partition and clustering interchangeably.} of $K$ where $p$ is the number of clusters and $P^1, \dots, P^{p}$ are the mutually exclusive and collectively exhaustive clusters where for each $1 \le i \le p $, $P^i \subseteq K$ and $\card{P^i} > 0$.
We will show results for the case of arbitrary clustering, as well as the case that the clustering forms a metrics space.   
%Need to add: We will mostly consider simple metrics, partitions, or clustering.

%Need to add: Regret, Weak Regret, Horizon $T$
For such an adversarial setting, the standard metric of interest is the \emph{regret} (aka weak regret in the original work of \cite{auer2002nonstochastic}), $R(T)$ defined for a time horizon $T$. To define it, we first need to determine the \emph{reward} of an algorithm $A$. 
%\revisionKDD{COMMENT (Roey): would it be better to first define the reward and only then introduce the regret?}
% Chen: I think it is OK. I emphsis reward.
For an algorithm $A$ that selects arm $a_t$ at time $t$ we consider its reward as:
%\begin{align*}
    $G_{A}(T) \eqdef \sum_{t=1}^{T} c_t(a_t)$. 
%\end{align*}
%
The best arm, in hindsight, is defined as %\gil{consider to add expectation}
%\begin{align*}
    $G_{\max}(T) \eqdef \max_{a \in K}\sum_{t=1}^{T} c_t(a)$,
%\end{align*}
and, in turn, the \emph{regret} of algorithm $A$ is defined as 
\begin{align}
\label{eq:Gmax}
    R(T) \eqdef G_{\max}(T) - G_{A}(T).
\end{align}
We will be interested in the expected regret $G_{\max}(T) - \exp[G_{A}(T)]$ \cite{auer2002nonstochastic}.

\begin{algorithm}[t]
\caption{\EXP32: Hierarchical Adversarial MAB Algorithm} %for AL-MAB}
\label{alg:exp32}
\begin{algorithmic}[1]
\revision{
\REQUIRE Parent A-MAB and Child A-MAB algorithms: Adversarial multi-armed bandits algorithms  (e.g., EXP3, Tsallis-INF), 
%\REQUIRE 
$k$ arms and their partition, $\cP$, into clusters 
\STATE \textbf{Initialize}: Init the parent A-MAB algorithm (weights, etc.). For each cluster, initialize its own child A-MAB algorithm (weights, etc.)
%set cluster weights $w_i(1)=1$ for $i \in K$. Set arms weights $w_{ij}(1)=1$ for all arms, $i \in K$, $1 \le j \le \card{P^i}$
\FOR{$t=1$ to $T$}
    \STATE (a) Cluster selection using the parent A-MAB algorithm on virtual arms (clusters). Let $P_t \in \cP$ be the selected cluster.
    \STATE (b) Arm selection using the child A-MAB algorithm on physical arms in $P_t$. Let $a_t \in P_t$ be the selected arm.
    \STATE (c) Pull $a_t$, observe reward $c_t$.
    \STATE (d) Update the arms-level parameters (e.g., weights) for physical arms in $P_t$ using $c_t$.
    \STATE (e) Update cluster-level parameters (e.g., weights) for virtual arms (clusters) using $c_t$.
    (e.g., For EXP3 using an importance-weighted estimator, $\hat{c}_t(P_t) = c_t(a_t) / \pi_t(P_t)$, where $\pi_t(P_t)$ is the probability of selecting cluster $P_t$).
\ENDFOR
}
\end{algorithmic}
\end{algorithm}

Next, we present our novel algorithm, \EXP32, and formally study its performance. 
%\section{The \EXP32 Algorithm}
%\label{s   ec:dblexp3}
The basic concept of the \EXP32 algorithm is based on a simple idea of \emph{divide-and-conquer}. Alongside the set of arms, we receive or create a partition (e.g., by using your favorite clustering algorithm) of the arms into clusters. This partition creates, in fact, a hierarchy. The first level is the ``clusters" level, where we can see each cluster as a virtual (adversarial)
%\gil{maybe mention, not here, that the cluster arm is always adversarial}) 
%\chen{did it in the correctness discussion, please check}
single arm based on the collective of arms in it. The second level is the level of physical arms within each cluster. See Algorithm \ref{alg:exp32} and Figure \ref{fig:system}.

\begin{figure*}[t!]
    \centering
    \includegraphics[width=\linewidth]{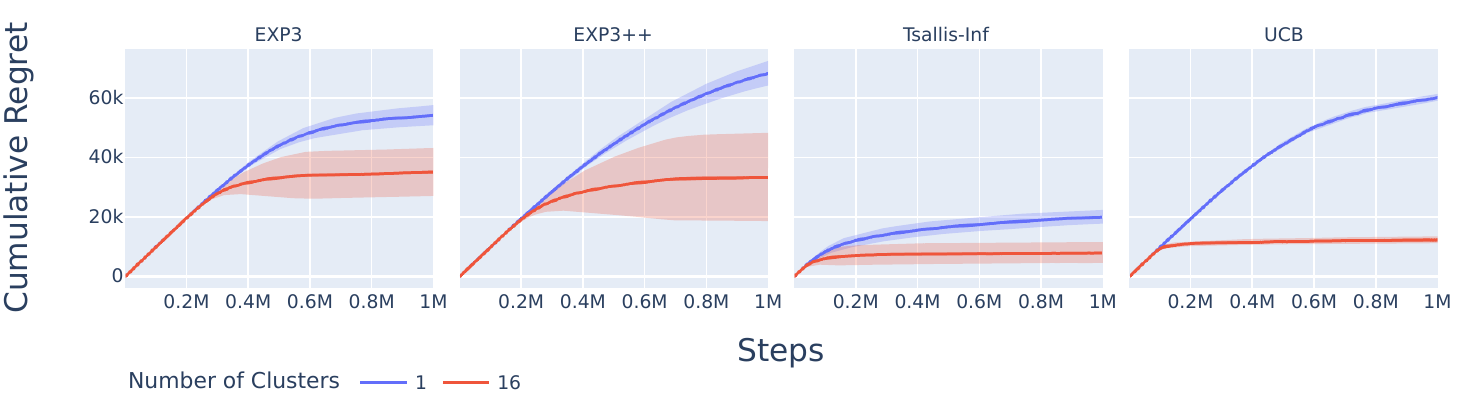}
    \caption{Comparing \EXP32 (with $\sqrt{k} = 16$ clusters) with ``flat'' algorithms (single cluster) for different well-known MAB algorithms. Stochastic Scenario from \cite{zimmert2021tsallis}, $k=256$.}
    \label{fig:4stochastic}
    \Description{Four side-by-side cumulative-regret curves vs.\ steps for ABoB versus flat baselines on a stochastic benchmark.}
\end{figure*}

\EXP32 uses this hierarchy to run flat Adversarial MAB (A-MAB) algorithms for each level, \revision{denoted as \emph{parent A-MAB} and \emph{child A-MAB}, for the first and second level, respectively. } The particular A-MAB algorithms used in \EXP32 can differ between levels, \revision{and the idea is to consider well-known A-MAB algorithms like EXP3 \cite{auer2002nonstochastic} and its variations \cite{seldin2017improved} or  Tsallis-INF \cite{zimmert2021tsallis} for example. } For the simplicity of presentation and as a concrete example, we use 
the classical EXP3 Algorithm \cite{auer2002nonstochastic} for both levels, unless otherwise stated.
%\gil{I think here we need to clarify, because EXP3 is not Lipschitz: Perhaps saying the EXP3 is Lipschitz within the cluster}.  \chen{removed the Lipschitz for now, will be introduced later.}
For the first level, there is a single EXP3, \revision{the parent A-MAB algorithm}, that on each time step, selects between the virtual arms \revision{generated } by the clusters and decides from which cluster it will sample the next arm. In turn, for each cluster, there is a second-level EXP3 algorithm that selects the next arm within the cluster, but \revision{it is activated only } when the first level selects that cluster.
Upon selecting an arm, we update the arm's reward and then the \revision{child algorithm parameters, like arms weights in the EXP3 and Tsallis-INF algorithms } (relative to the arms in the current cluster).  Next, we update the reward of the virtual arm of the cluster  \revision{and the parent algorithm's parameters } 
(relative to other virtual arms) \revision{and move to the next time step.
}
% %for the virtual arm of the cluster.
%\paragraph{Correctness and Time Complexity of \EXP32.}

\para{Correctness and Time Complexity of \EXP32.}
\revision{The correctness of the algorithm follows from the observation that each cluster can be viewed as an adversarial virtual arm, that generates a reward when activated at time $t$. Therefore, since the parent A-MAB algorithm is adversarial, the algorithm setup is valid. We note that regarding the child A-MAB algorithm, we have more flexibility, and if we know, for example, that each cluster is stochastic, we can use algorithms that may better fit this environment, like UCB \cite{ucb2002auer}}. 
\revision{An important feature of \EXP32 is its running time. In the adversarial setting the running time of many algorithms is $\cO(kT)$, where $T$ updates, each in the order of $\cO(k)$ are required (e.g., in EXP3). In \EXP32, however, we still need $T$ updates, but each is of the order of $\cO(\textrm{\# of clusters + \# of arms in the selected cluster})$, this for example can be significantly lower then $\cO(k)$, e.g., when we use $\sqrt{k}$ clusters each of size $\sqrt{k}$ (up to rounding). }

The main question we address next is: In terms of the regret, how much do we lose or gain by using \EXP32 and other algorithms, and the hierarchical approach? 
We start with an empirical study and experimental results on a real system.
Our results clearly demonstrate the value of the approach compare to state-of-the-art algorithms.  In turn, in Section \ref{sec:theory} we provide \revisionKDD{an} analysis of the approach, proving 
that in the worst-case we can't lose too much, but under good partitions with certain Lipschitz condition, we can prove the benefit of our algorithm.

\section{Empirical Study}\label{sec:empirical}
\revision{In this section (and the appendix), we report on an extensive empirical study on the performance of \EXP32 compared to flat algorithms. We consider syntactic scenarios in increasing complexity, and in Section \ref{sec:experimental_results}, we present results based on real system measurements.
We report results on three adversarial algorithms Tsallis-INF \cite{zimmert2021tsallis}, EXP3 \cite{auer2002nonstochastic}, and EXP3++ \cite{seldin2017improved}, but also on the, well-known, UCB \cite{ucb2002auer} algorithm for the stochastic case. In turn, we consider different \EXP32 algorithms, including the above MAB algorithms. We mostly concentrate on the case where both the parent and child algorithms are the same, and in particular on the Tsallis-INF algorithm, which is the state-of-the-art algorithm enjoying the best-of-both-worlds regret bounds (i.e., it is optimal both for stochastic and adversarial settings). 
Unless otherwise stated, we consider Bernoulli \revisionKDD{r.v.} arms with different settings on their mean values. The default number of arms is $k=256$, and the default number of steps is $T=10^6$. We repeated each experiment 10 times and report the average and standard deviation. Additional and more extensive figures, including non-identical parent-child algorithms, are given in the Appendix.}
%\ref{apn:detailed}} 
%\revisionKDD{COMMENT (Roey) The reference is broken but I'm not sure which appendix part was supposed to be here, is it just the general appendix section?}.

%\subsection{Stochastic Scenario}
%\paragraph{Stochastic Scenario.}
\para{Stochastic Scenario.}
\revision{The first experiment, shown in Figures \ref{fig:4stochastic} 
%(and Figure \ref{fig:4stochastic} in Appendix) 
is a standard stochastic MAB setting \cite{zimmert2021tsallis}, where the mean rewards are $\frac{1 + \Delta}{2}$ for the single optimal arm and $\frac{1 - \Delta}{2}$ for all the suboptimal arms, where $\Delta=0.1$. Other values for $\Delta$ and $k$ are reported in the appendix.
We can clearly observe the performance improvement of \EXP32 in all Algorithms for
$p=\sqrt{k}=16$ clusters each of size $\sqrt{k}=16$. As was reported, for a single cluster, Tsallis-INF produces the best results as in \cite{zimmert2021tsallis}, but \EXP32 using Tsallis-INF with $\sqrt{k}=16$ further improves the regret. 
}

%\paragraph{Fixed Optimal Arm and Nonstochastic (adversarial) Scenario.}
%\subsection{Nonstochastic (adversarial) Scenario}
\para{Fixed Optimal Arm and Nonstochastic (adversarial) Scenario.}
\revision{The second experiment, also taken from \cite{zimmert2021tsallis}, considers a non-stochastic (adversarial) environment with a single fixed optimal arm but changing mean values. The mean reward of (optimal arm, all sub-optimal arms) switches between $(\Delta, 0)$ and $(1, 1-\Delta)$, while staying unchanged for phases that are increasing exponentially in length. We set $\Delta=0.1$ and other values for $\Delta$ and $k$ are reported in the appendix. Figure \ref{fig:4staticArmNonstochastic} 
%(and Figure \ref{fig:4staticArmNonstochastic} in the Appendix)
presents the results for several MAB algorithms, where Tsallis-INF again achieves the best results as in \cite{zimmert2021tsallis}, and \EXP32 improves it.}

\begin{figure*}
    \centering
    \includegraphics[width=\linewidth]{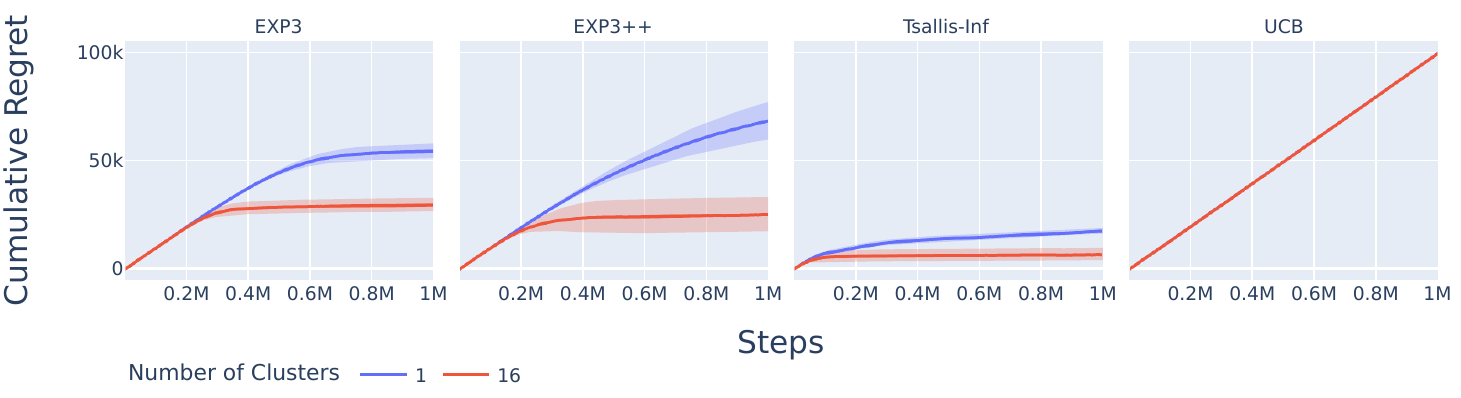}
    \caption{Fixed Optimal Arm, and Nonstochastic Scenario from \cite{zimmert2021tsallis}, $k=256$. Comparing \EXP32 (with $\sqrt{k} = 16$ clusters) with ``flat'' algorithms (single cluster) for different well-known MAB algorithms.}
    \label{fig:4staticArmNonstochastic}
    \Description{Four side-by-side cumulative-regret curves vs.\ steps for ABoB versus flat baselines on a nonstochastic benchmark.}
\end{figure*}

% \begin{figure}
%     \centering
%     \includegraphics[width=.9\linewidth]{figures/regret_comparison_same_delta_env_adversarial_2.pdf}
%     \caption{Fixed Optimal Arm, and Nonstochastic Scenario, $k=256$. Comparing \EXP32 (with $\sqrt{k} = 16$ clusters) with ``flat'' algorithms for EXP3++ and Tsallis-INF algorithms.}
%     \label{fig:2staticArmNonstochastic}
% \end{figure}

%\subsection{``Traveling'' Optimal Arm, Nonstochastic (adversarial) and Metric Spaces}\label{sec:traveling}
%\paragraph{``Traveling'' Optimal Arm, Nonstochastic (adversarial) and Metric Spaces.}\label{sec:traveling}
\para{``Traveling'' Optimal Arm, Nonstochastic (adversarial) and Metric Spaces.}
\revision{In the third experiment, we consider a scenario motivated by our application for configuration tuning. We consider a Nonstochastic (adversarial) reward with changing optimal arms (i.e., ``traveling arms''), but all rewards are in a metric space and, in particular, follow a Lipschitz condition. }
%
%
%We have compared the performance of both EXP3 and \EXP32 \revision{using EXP3 } in several simulated environments to validate the theoretical results in Section \ref{sec:theory}. We have conducted both one and two-dimensional setups with stochastic reward (not dynamic), nonstochastic reward schedules, and replay of traces of a real-time storage system (see Section \ref{sec:experimental_results}).
%
%\subsection{Simulation Setup}\label{sec:sim_setup}
%\paragraph{Environment Setup.}
More formally, we have used the following setup for the environment:
(i) Metric Space: The metric space is defined as a hypercube of dimension $d$ $\mathcal{Q}=[0,\frac{1}{\sqrt{d}}]^d$. We use a setup such that $\sqrt[d]{k}$ is an integer.
(ii) Arms' location: The $k$ arms are placed over an equally spaced grid such that the distance between two closest points in each dimension is \revisionKDD{$w=\frac{1}{\sqrt{d} \sqrt[d]k}$}. For arm $i$, we denote its location on the grid as $x_i$ and recall that its \revision{mean } reward at time $t$ is denoted as $c_t(x_i)$ ~(or $c_t(i)$).
(iii) Arm's reward distribution: We define the mean reward to be $c_t(x_i) = 1 - \| a^*(t) - x_i \|$, where $a^*(t)$ is a point in the (continuous) cube that represents the best $d$-dimensional configuration settings at time $t$. Note that $a^*(t)$ is a parameter of the optimization problem and could be \emph{dynamic} over time.
(iv) Non-Stochastic: to simulate a dynamic system, we have changed $a^*(t)$ over time using a normal random walk, such that $a^*(t+1)=a^*(t)+n(t)$ and $n(t)$ is a $d$-dimensional normally distributed random variable with zero mean and equal diagonal covariance $\sigma^2$ (keeping $a^*(t)\in\mathcal{Q}$ by clipping).
%See \autoref{sec:appendix} for other settings.

%\paragraph{EXP3 and \EXP32 Setup.}
%There are several important parameters we need to set up so the \EXP32 algorithm will be well-defined.
%First is the partition or clustering we use. 
Unless otherwise stated, we clustered the $d$-dimensional cube into equal volume sub-spaces.
For \EXP32 we evaluated an increasing number of clusters $p$, from $p=1$ to $p=k$, note that for these extreme cases it coincides with the flat A-MAB algorithm.
%
%EXP3 uses an important parameter $\gamma$, which controls exploration, was selected according to the literature recommended one: $\gamma(k,T)=\min(1,\sqrt{\frac{klog(k)}{(e -1)T}}$). For \EXP32 we chose the first level bandit gamma to be $\gamma(p,T)$ corresponding to the number of partitions, whereas for the secondary bandits we chose $\gamma(\frac{k}{p},\frac{T}{p})$. We noticed that the results are robust to similar choices of $\gamma$, but leave additional optimizations for future research.
%
%
%\paragraph{Experiments Setup.}
The default setting we used was \(k=256\) arms, and the clusters varied as \(p=2^j\), where \(j \in \{0, 1, \dots, \log_2{k}\}\). 
%

%\paragraph{Nonstochastic, Adversarial Example.}
In order to further validate the theoretical results on a nonstochastic, adversarial setup, we have changed the system parameter $a^*(t)$ over time such that the different arms' mean reward changes over time, and thus also the best arm. In the appendix, Figures \ref{fig:random_walk_heatmap} and \ref{fig:random_walk_best} depicts this by showing the arms' mean reward over time as well as the optimal arm index that changes over time and the location of the best arm in the parameter space and how often an arm was the best.
Figure \roey{\ref{fig:2D_non_time}} presents the results of the above setup comparing the regret of \EXP32 with Tsallis-INF \roey{as a function of number of steps. Additionally, Figure \ref{fig:2D_non_clusters} shows the regret as a function of the number of clusters, which shows that using 16 clusters, \EXP32 achieves a regret of }$\roey{435\pm69}$ compared to the flat Tsallis-INF baseline \roey{which achieves a regret of }$\roey{4879\pm259}$, again showing the potential gain (t-test p-value $\roey{8.2\times 10^{-14}}$), of about \roey{91}\%.
\revision{Notice that the regret can be negative since we compare the results of Tsallis-INF to the best \emph{fixed} arm while the optimal arm is traveling. Since Tsallis-INF can ``travel" as well, its result can be better than the fixed optimal.}

\para{Running times} 
\revisionKDD{Figure \ref{fig:run-times} validates our assertion about the running time of \EXP32 vs. a "flat" algorithms. When evaluated on Tsallis-INF, the best fit for the running time of \EXP32 was $41.6+0.08\sqrt{k}$, with number of cluster $\sqrt{k}$, while the best fit for "flat" algorithm running time is $20.9+0.034k$, as we conjectured. Both fits are statistically significant, $p$-values $< 10^{-8}$.}

%p-value
%p-value

\begin{figure}
    \centering
    \includegraphics[width=.9\linewidth]{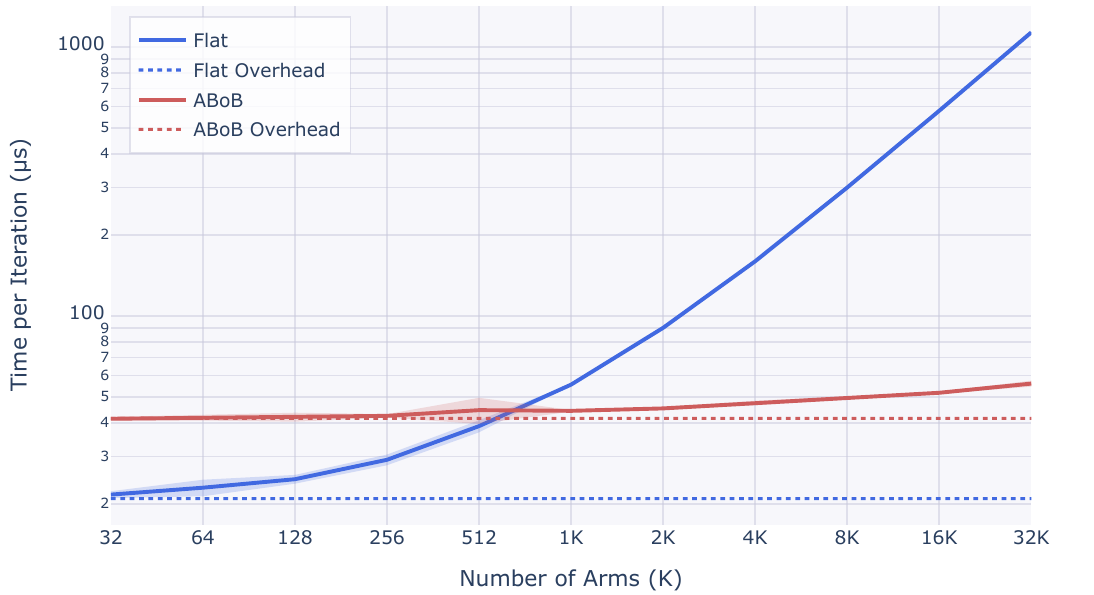}
    \caption{The running times of \EXP32 compare to a "flat" Tsallis-INF for increasing numbers of arms $k$.}
    \label{fig:run-times}
    \Description{Line plot of running time versus number of arms for ABoB and flat Tsallis-INF.}
    \vspace{-.5cm}
\end{figure}

\begin{figure*} % float at top
  \centering
  % --- First figure (with two subfigures) ---
  %\begin{minipage}[t]{0.47\textwidth}
    \centering
    \subfigure[Regret vs. Steps]
    {\label{fig:2D_non_time}
    \includegraphics[width=.24\linewidth]{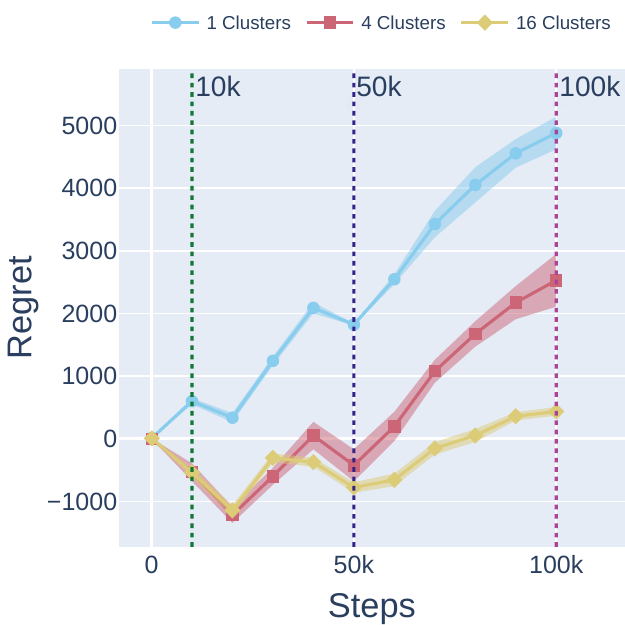}}
    \subfigure[Regret vs. Clusters]
    {\label{fig:2D_non_clusters}
    \includegraphics[width=.24\linewidth]{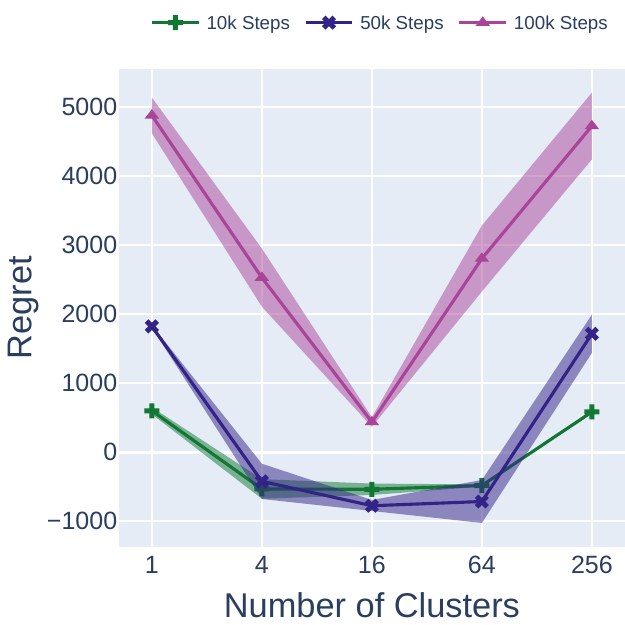}}
    %\vspace{-10pt}
    \subfigure[``Traveling'' Arm]{
    \label{fig:random_walk_best}
    \includegraphics[width=.26\linewidth]{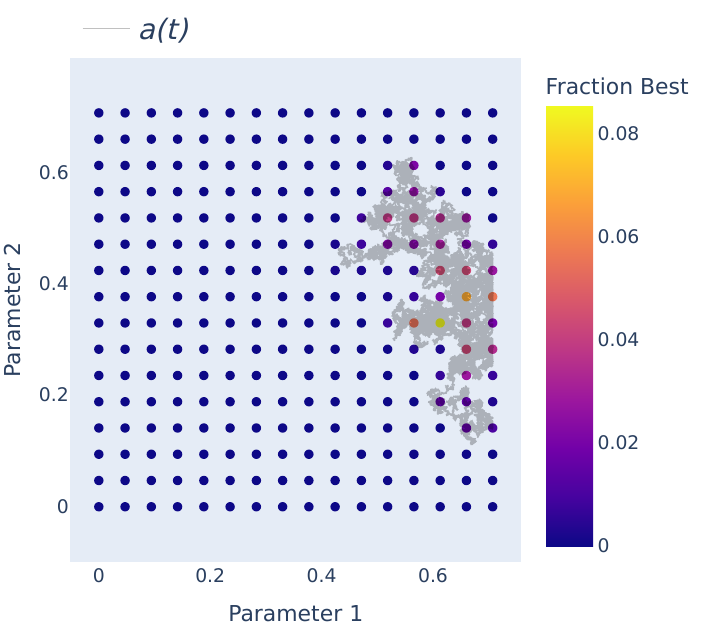}}
    \subfigure[Over time]{
    \label{fig:random_walk_heatmap}
    \includegraphics[width=.22\linewidth]{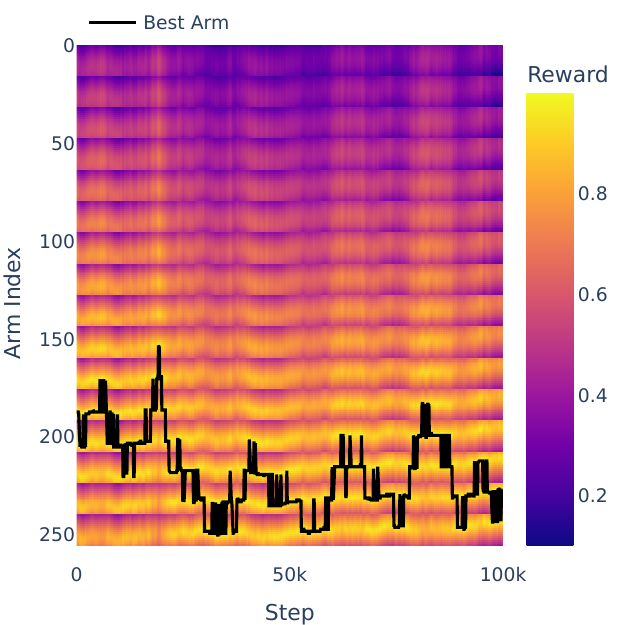}}   
    \caption{\revision{The ``Traveling'' Optimal Arm, Nonstochastic and Metric Space Scenario. (a) The cumulative regret for ABoB as a function of steps. (b) The cumulative regret for ABoB as a function of the number of clusters.
    (c) The optimal arm's travel in the configuration space and optimality frequency. (d) The arms’ mean reward over time and the optimal arm index.}
    %that changes over time.}
    }
    \label{fig:simNonstResults}
%  \end{minipage}
    \Description{Four-panel plot: regret vs.\ steps, regret vs.\ clusters, optimal-arm trajectory in the configuration space, and a heatmap of arm rewards over time.}
    \vspace{-.2cm}
\end{figure*}

\section{Experimental Results: A Real System}
\label{sec:experimental_results}
%\vspace{-3pt}
We compare a flat and \EXP32 algorithms using Tsallis-INF on a real storage system with the goal of maximizing the performance of the system. The system has a large configuration space of several parameters, some continuous and some discrete (integer value), where each performance evaluation takes a few seconds. Furthermore, the system \revisionKDD{state undergoes dynamically changing workloads (e.g., number of active jobs, as new jobs come and go), which affect the observed reward. However, for every state, the Lipschitz condition still holds (Eq. \eqref{eq:Lipschitz}) and a minor change of a parameter results in a minor shift in the outcome.}
%The set of potential arms is extensive since we can tune each parameter to many possible values. In turn, we assume a smoothness condition on each parameter, such that each minor change of a parameter results in a minor shift in the outcome. The ``state'' of the system (e.g., number of active jobs) is unknown and can change over time (e.g., as new jobs come and go), even dramatically, but for every state, the Lipschitz condition holds (Eq. \eqref{eq:Lipschitz}).
Here we report results on optimizing $2$ parameters while keeping the others fixed.
\revisionCR{While our theoretical results are dimension-independent, focusing on two parameters here demonstrates that the hierarchical framework is viable in a noisy, constrained, real environment, allows clear 2-D visualizations of the clustering and decision boundaries (Figure~\ref{fig:snap_reward}), and keeps the arm count compatible with the seconds-per-evaluation cost of the real system, while still capturing a frequent practical scenario (tuning two parameters with others fixed). Scalability with the number of arms is a challenge for both flat and hierarchical approaches; our runtime analysis (Figure~\ref{fig:run-times}) indicates that \EXP32 scales more favorably than the flat baseline.}

Each run consists of \roey{$100,000$} iterations with 256 arms generated uniformly at random from the system's configuration search space. Clustering was implemented using K-means over the normalized arms parameters. The system cycles through six distinct workloads, switching approximately every 10,000 iterations.
%(See 'Replaying Real-system data' paragraph for more about details).

% \begin{figure}
%     \centering
%     %height
%     \subfigure[``Traveling'' Arm]{
%     \label{fig:random_walk_best}
%     \includegraphics[width=.49\linewidth]{figures/random_walk_best_arm.pdf}}
%     \subfigure[Over time]{
%     \label{fig:random_walk_heatmap}
%     \includegraphics[width=.44\linewidth]{figures/random_walk_heatmap.pdf}}
%     \caption{\revision{The ``Traveling'' Optimal Arm, Nonstochastic (adversarial) and Metric Space  Scenario. (a) Example of the travel of the optimal arm in the configuration space and how often an arm was the best. (b) The arms’ mean reward over time and the optimal arm index that changes over time.}}
%     \label{fig:simNonstDetails}
%     \Description{}
% \end{figure}

%\paragraph{Rewards in a Real System.}
\para{Rewards in a Real System.}
To validate that the real system follows the Lipschitz condition, we estimate the Lipschitz's constant $\ell$ in the real arm reward mean and compare it to a shuffled reward distribution. To estimate $\ell$, we iterated over the arms, and for each arm $x_i$ over 
its $n$ nearest neighbors, $x_j \in \mathcal{N}_n(i)$. For each arm, we then computed the mean ratio between the rewards and the metric: 
$\ell_i = \frac{1}{n}{\sum_{j\in \mathcal{N}_n(i)}}(|r_j-r_i|)/(|x_j-x_i|)$.
%In the Appendix, 
Figure \ref{fig:lipshitz} shows the distribution of $\ell_i$ for both the real rewards shown in Figure \ref{fig:snap_reward} and a random permutation of the rewards. The figure validates that the Lipschitz assumptions hold for the real system as the distribution of $\ell_i$ is highly concentrated near zero. % (statistical test).

% \begin{figure}
%     \centering
%     \includegraphics[width=.5\linewidth]{figures/snap_lipfshitz.pdf}
%     \caption{Distribution of the estimation of the Lipschitz constant for the empirical reward function vs.e shuffled one.}
%     \label{fig:lipfshitz}
% \end{figure}

%\autoref{fig:snap_Lipschitz}: show the Lipschitz condition in real system (missing for now)

% EXP3 vs. \EXP32: 
%\paragraph{Flat Tsallis-INF vs. \EXP32 with Tsallis-INF: Real system data.}\label{sec:replay}
\para{Flat vs. \EXP32 (Tsallis-INF): Real system data.}
Using the setup described above (for $T=\roey{100}$K and the number of arms is $k=256$), we run the algorithms on the \emph{real system}. \roey{The star-shaped points in Figure \ref{fig:2D_replay_clusters} shows that Tsallis-INF obtained a regret of 7819, while \EXP32, using 16 clusters, provided a regret of 4025. This is an improvement of about 49\%}. The regret was computed by comparing the algorithms' rewards to the best empirical arm in hindsight, where the rewards of each arm are interpolated between measured samples.
To validate the approach further, we replayed the reward sequence that was recorded from the real system (again using interpolation to fill rewards in all time steps). Figures \roey{\ref{fig:2D_replay_time}} and \ref{fig:2D_replay_clusters} shows how the \EXP32 $(\roey{5543\pm21})$ again is dominant over the flat Tsallis-INF baseline ($\roey{7584\pm79}$) (t-test p-value $\roey{1.8\times 1.3^{-15}}$), about \roey{27}\% improvement. 

\section{Analytical Bounds for {\EXP32} Alg.}
\label{sec:theory}

%\chen{Introduce the partition (not even a metric) and the hierarchal algorithm}

\revision{In this section we study \EXP32 analytically. }  
\revision{To continue, we need additional notations. }
Recall that $T$ is the time horizon, $\cP$ is a partition \revision{of the set of arms $K$, $k$ is the number of arms, } and $p$ is the number of clusters. Let $T^i$ be the ordered set of times by which the algorithm visits cluster $i$, so
$\sum \card{T^i} = T$. We denote by $t(i,j) \in T^i$, the time in which the algorithm visited cluster $i$ for the $j$th time.
We then additionally define $a^* \in P^*$ as the best arm in hindsight, i.e., the arm that maximizes Eq. \eqref{eq:Gmax}, and $P^*$ as the cluster that contains $a^*$.\footnote{If $a^*$ is not unique, we consider $P^*$ to be the largest cluster that contains an $a^*$.} 
Let $k^* = \card{P^*}$ denote the size of the cluster that holds $a^*$.
% We define $P^{+}$ as the best cluster in hindsight, the cluster that should have been played constantly by the first-level EXP3, but internally by its own EXP3 algorithm. Formally, it is the cluster that \emph{maximizes}, 
% \begin{align*}
%     G_{\mathrm{maxP}}(T) &\eqdef \max_{i \in K} \sum_{t=1}^{T} c_t(P^i).
% \end{align*}
Let $G_{\mathrm{maxP}}(T)$ be the \revision {expected}~ reward of the best cluster in hindsight:
\begin{align*}
    G_{\mathrm{maxP}}(T) &\eqdef \max_{i \in K} \sum_{t=1}^{T} c_t(P^i) = 
    \sum_{t=1}^{T} c_t(P^{+}),
\end{align*}
and $P^{+}$ be the best cluster in hindsight, the cluster (virtual arm) that should have been played constantly by the first-level, \revision{parent A-MAB algorithm (e.g., EXP3), } assuming it is played internally by its own second-level, \revision{child A-MAB algorithm (e.g. EXP3).}

% and let $P^{+}$ be the best cluster in hindsight, the cluster that should have been played constantly by the first-level EXP3, but internally by its own EXP3 algorithm. Formally, it is the cluster that \emph{maximizes} Eq. \eqref{eq:p_plus}.
Lastly, for each cluster $P^i$, we define its best \revision{expected reward~}  in hindsight $G_{\mathrm{maxP}^i}$; what was its \revision{expected reward } if we were playing its best arm in hindsight \revision{(at the times we visited it)}, formally,
\begin{align*}
    G_{\mathrm{maxP}^i}(T^i) &\eqdef \max_{a \in P^i} \sum_{j=1}^{\card{T^i}} c_{t(i,j)}(a).
\end{align*}

% \begin{align}
%     %G_{\max}(P,T) &= \max_{a \in P} \sum_{t=1}^T c_t(a)\\
%     G_{\max}(T) &= \max_a \sum_{t=1}^T c_t(a) = \max_{P \in \cP} G_{\max}(P,T)\\
%     a^*(T) &= \arg\max_a \sum_{t=1}^T c_t(a)\\
%     a^* &\in P^* \\
%     %G_{\max} &= \sum_{t=1}^T r^*_t(C^*) \\
%     G_{\mathrm{maxP}^i}(T^i) &= \max_{a \in P^i} \sum_{j=1}^{T^i} c_{t(i,j)}(a) \\
%     G_{\mathrm{maxP}}(T) &= \max_{i \in K} \sum_{t=1}^{T} c_t(P^i)\\
%     P^{+}(T) &= \arg\max_{i \in K} \sum_{t=1}^{T} c_t(P^i)
% \end{align}

In a hierarchical multi-armed bandit \revision{setup of the \EXP32 algorithm}, the total regret can be expressed as the sum of two components:
(i) Regret for choosing a cluster other than the best one: This is the regret incurred by the first-level bandit in not selecting the best cluster, defined as the cluster containing the best arm in hindsight.
(ii) Regret within the cluster: This is the regret incurred by the second-level bandit in selecting an arm within the chosen cluster relative to the best arm in that cluster.
\revision{ Next, we use this observation to bound the regret of \EXP32 in different scenarios.}

%\subsection{Arbitrary Clustering: Not Much to Lose, Even in the Worst Case}
\subsection{Arbitrary Clustering: Not Much to Lose}\label{sec:notmuch}

In this subsection, we study the regret where the partition (or clustering) is done arbitrarily, i.e., without assuming any metric between the arms. For EXP3, the analysis shows (Theorem \ref{thm:clusters}) that even in the case of arbitrary partition of the arms, the asymptotic regret is equivalent in the worst-case to those achieved by the ``flat'' EXP3 algorithm, i.e., $\cO\left(\sqrt{kT\log k}\right)$ \cite{auer2002nonstochastic}.

The expected regret of \EXP32 can be bounded as follows: For each level of the hierarchy, we analyze its regret resulting from the relevant A-MAB algorithms. At the first level, the regret is of the parent A-MAB playing with $p$ virtual arms of the clusters, and with respect to $G_{\max}$. At the second level, we compute the regret within each cluster $P^i$ using its child A-MAB and with respect to the best cost of the cluster in hindsight $G_{\mathrm{maxP}^i}$. \revision{Formally,} 
% \begin{align}
%     G_{\max}(T) - G_{\EXP32} &= G_{\max}(P^*, T) - \sum_{t=1}^{T} c_t(a_t)\\
%     &= G_{\max}(P^*, T) - \sum_{i=1}^{p } G_{\mathrm{maxP}^i}(T^i) + \sum_{i=1}^{p } G_{\mathrm{maxP}^i}(T^i) - \sum_{t=1}^{T} c_t(a_t) \\
%     &= G_{\max}(P^*, T) - G_{\mathrm{maxP}}(T) + G_{\mathrm{maxP}}(T) - \sum_{i=1}^{p } G_{\mathrm{maxP}^i}(T^i) \\
%     &~~~+ \sum_{i=1}^{p } G_{\mathrm{maxP}^i}(T^i) - \sum_{t=1}^{T} c_t(a_t) \\
%     % &= G_{\max}(P^*, T) - G_{\max}(P^{+}, T) + G_{\max}(P^{+}, T) - G_{\mathrm{maxP}}(T) \\
%     % &~~~+ G_{\mathrm{maxP}}(T) - \sum_{i=1}^{p } G_{\mathrm{maxP}^i}(T^i) 
%     % + \sum_{i=1}^{p } G_{\mathrm{maxP}^i}(T^i) - \sum_{t=1}^{T} c_t(a_t) \\
%     &= G_{\max}(P^*, T) - G_{\mathrm{maxP}}(T) \\
%     &~~~+ G_{\mathrm{maxP}}(T) - G_{\max}(P^{+}, T) \\
%     &~~~+ G_{\max}(P^{+}, T) - \sum_{i=1}^{p } G_{\mathrm{maxP}^i}(T^i) \\
%     &~~~+ \sum_{i=1}^{p } G_{\mathrm{maxP}^i}(T^i) - \sum_{t=1}^{T} c_t(a_t)
% \end{align}

%\begin{enumerate}
%\item 
{\em - First level}: the regret of not choosing the best cluster in hindsight $P^{+}$ during all times up to $T$:
{\small
\begin{align*}
    R_1^a \eqdef G_{\mathrm{maxP}}(T) - \sum_{t=1}^T c_t(P_t) = G_{\mathrm{maxP}}(T) - \sum_{i=1}^{p }\sum_{j=1}^{\card{T^i}} c_{t(i,j)}(P^i),
\end{align*}
}
%\item 
plus the regret for not choosing the cluster $P^*$ and selecting within always arm $a^*$:
{\small
\begin{align*}
       R_1^b \eqdef G_{\mathrm{max}}(T) - G_{\mathrm{maxP}}(T)
\end{align*}
}
%\item 
{\em - Second level}: the regret in each cluster not playing its best arm in hindsight all the time:
{\small
\begin{align*}
        R_2 \eqdef \sum_{i=1}^{p } \left( G_{\mathrm{maxP}^i}(T^i) - \sum_{j=1}^{\card{T^i}} c_{t(i,j)}(a_{t(i,j)}) \right)
\end{align*}
}
%\end{enumerate}
Formally, we can claim the following.

\begin{claim}\label{clm:nolose}
    For the \EXP32 algorithm that uses EXP3, the following holds:
\begin{enumerate}
\renewcommand{\labelenumi}{(\arabic{enumi})}
    \item \label{itm:1a} $\exp\left[R_1^a\right] \le \cO\left(\sqrt{p T \log p }\right)$.
    \item \label{itm:1b} $\exp\left[R_1^b\right] \le\cO\left(\sqrt{k^* T \log k^*}\right)$.
    \item \label{itm:2} $\exp\left[R_2\right] \le\cO\left(\sqrt{k T \log (k/p)}\right)$.
\end{enumerate}
    
% \begin{enumerate}
% \item 
% {\small
% % \begin{align}
% %     \exp\left[G_{\mathrm{maxP}}(T) - \sum_{t=1}^T c_t(P_t)\right] \le \cO\left(\sqrt{p T \log p }\right)
% % \label{eq:claim1}
% \begin{align}
%     \exp\left[R_1^a\right] \le \cO\left(\sqrt{p T \log p }\right)
% \label{eq:claim1}
% \end{align}
% %\item 
% \begin{align}
%     % \exp\left[G_{\mathrm{max}} - G_{\mathrm{maxP}}(T)\right] \le\cO\left(\sqrt{k^* T \log k^*}\right)
%     \exp\left[R_1^b\right] \le\cO\left(\sqrt{k^* T \log k^*}\right)
% \label{eq:claim2}
% \end{align}
% %\item 
% \begin{align}
%     % \exp\left[\sum_{i=1}^{p } \left( G_{\mathrm{maxP}^i}(T^i) - \sum_{j=1}^{\card{T^i}} c_{t(i,j)}(a_{t(i,j)}) \right)\right] 
%     % \le\cO\left(\sqrt{k T \log \frac{k}{p}}\right)
%     \exp\left[R_2\right] 
%     \le\cO\left(\sqrt{k T \log \frac{k}{p}}\right)
% \label{eq:claim3}
% \end{align}
% }
% \end{enumerate}
\end{claim}

\begin{proof}[Proof of Claim \ref{clm:nolose}]
    The bound in \eqref{itm:1a} follows from running EXP3 on the clusters, each as an adversarial virtual arm. There are $p$ clusters, and the regret is with respect to playing the best cluster (virtual arm), $P^{+}$ in hindsight, so the results directly follow from the EXP3 bound.  
    For the bound in \eqref{itm:1b}, we note that,
\begin{align*}
    &\exp[G_{\mathrm{max}}\revisionKDD{(T)} - G_{\mathrm{maxP}}(T)] = G_{\mathrm{max}}\revisionKDD{(T)} - \exp[\sum_{t=1}^{T} c_t(P^{+})] \le \\
    &\le G_{\mathrm{max}}\revisionKDD{(T)} - \exp[\sum_{t=1}^{T} c_t(P^{*})]  
    \le\cO\left(\sqrt{k^* T \log k^*}\right), 
\end{align*}    
where the first inequality follows since cluster $P^{+}$ has the largest expected reward (by definition), at least as large as cluster $P^{*}$, and the last inequality follows from the EXP3 bound in the cluster $P^*$ for which $a^* \in P^*$. 
For the equation \eqref{itm:2} in the claim, the result follows from running EXP3 in each cluster,
{\small
\begin{align}
        &\exp\left[\sum_{i=1}^{p} \left( G_{\mathrm{maxP}^i}(T^i) - \sum_{j=1} \underset{t'=t(i,j)}{c_{t'}(a_{t'})}\right ) \right] \notag \le \\
        %&\le \label{eq:R2}  \\
        &\le \sum_{i=1}^{p }\cO\left(\sqrt{\card{P^i} \card{T^i} \log \card{P^i}}\right) 
        \le\cO\left(\sqrt{k T \log \frac{k}{p} }\right), \label{eq:R2}
\end{align}}
where the first inequality is the EXP3 bound, and the last inequality follows from the concavity of the function and under the constraints
$\sum \card{\cP^i} = k $ and $\sum \card{T^i} = T$, the worst case is $\forall i, \card{\cP^i}= \frac{k }{p }$ and $\card{T^i} = \frac{T}{p}$.
\end{proof}

Based on Claim \ref{clm:nolose}, we can bound the regret of \EXP32.

\begin{theorem}\label{thm:clusters}
    For $k$  nonstochastic arms, $T>0$ and a partition $\cP$ of the arms, the regret of the \EXP32 algorithm using EXP3 is bounded as follows:
    \begin{align*}
        G_{\max} - \exp[G_{\EXP32}] &\le 
        \cO\left(\sqrt{p T \log p }\right) + \cO\left(\sqrt{k^* T \log k^*}\right) \\
        &\;\;\;+\cO\left(\sqrt{k T \log (k/p)}\right)
    \end{align*}
where $p$ is the number of clusters and $k^*$ is the size of the cluster with the best arm in hindsight.
\end{theorem}

Since $p, k^* \le k$, the worst case bound is in the same order as the flat case. Moreover, we will usually use $p, k^* \ll k$, and the contribution of the first two terms should be smaller than the third. With this, we can state the first main takeaway of our approach. 

\paragraph{{\small $\blacksquare$} Takeaway.} \textit{Using clustering does not ``hurt'' much the overall regret of the \emph{flat} approach.} 

\revision{We note that similar results can be obtained for other A-MAB algorithms like 
Tsallis-INF \cite{zimmert2021tsallis}, removing the logarithmic factors. 
%See Appendix ??. 
Still, the question remains if we can benefit from clustering; the following subsection answers this question affirmatively.}

% \begin{align}
% G_{\max} - \exp[G_{\text{EXP3}^2}] & =G_{\max} - \sum_{t=1}^T r_t(a_t) \\
%     &= G_{\max} - \sum_{t=1}^T r_t(a_t) + G_{\mathrm{maxP}} - G_{\mathrm{maxP}} \\
%     &= G_{\max} - G_{\mathrm{maxP}} + G_{\mathrm{maxP}} - \sum_{t=1}^T r_t(a_t) \\
%     &= G_{\max} - G_{\mathrm{maxP}} + G_{\mathrm{maxP}} - \sum_{t=1}^T r_t(C_t) \\
%     &= G_{\max} - G_{\mathrm{maxP}} +\cO(\sqrt{c T \log c}) \\
%     &= \sum_{t=1}^T (r^*_t(C^*) - r_t(C^{+})) +\cO(\sqrt{c T \log c}) \\
%     &\le  \sum_{t=1}^T (r^*_t(C^*) - r_t(C^{*})) +\cO(\sqrt{c T \log c}) \\
%     &=\cO(\sqrt{\frac{k}{c} T \log \frac{k}{c}} +\cO(\sqrt{c T \log c})
% \end{align}

\begin{figure*} % float at top
  % --- Second figure (with two subfigures) ---
%  \begin{minipage}[t]{0.47\textwidth}
    \centering
    \subfigure[Regret vs. Steps]
    {\label{fig:2D_replay_time}
    \includegraphics[width=.24\linewidth]{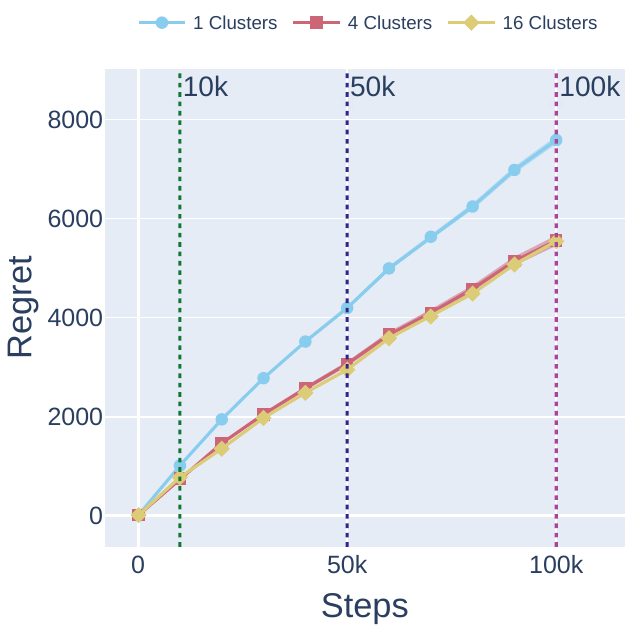}}
    \subfigure[Regret vs. Clusters]
    {\label{fig:2D_replay_clusters}
    \includegraphics[width=.24\linewidth]{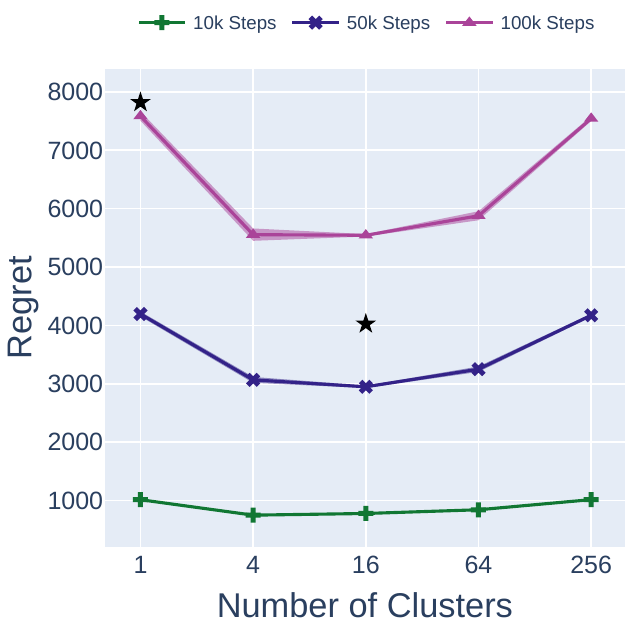}}
    %\vspace{-10pt}
    \subfigure[Empirical Lipschitz]
    {\label{fig:lipshitz}
    \includegraphics[width=.24\linewidth]{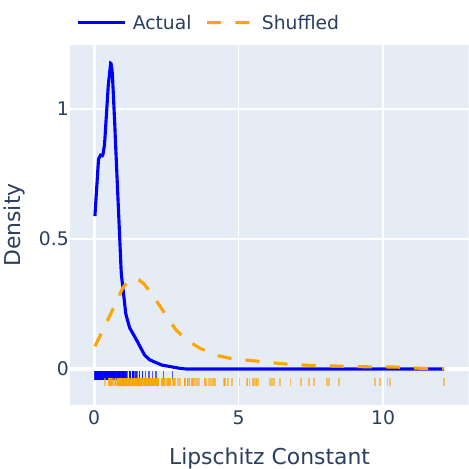}}
    \subfigure[Empirical mean]
    {\label{fig:snap_reward}
    \includegraphics[width=.24\linewidth]{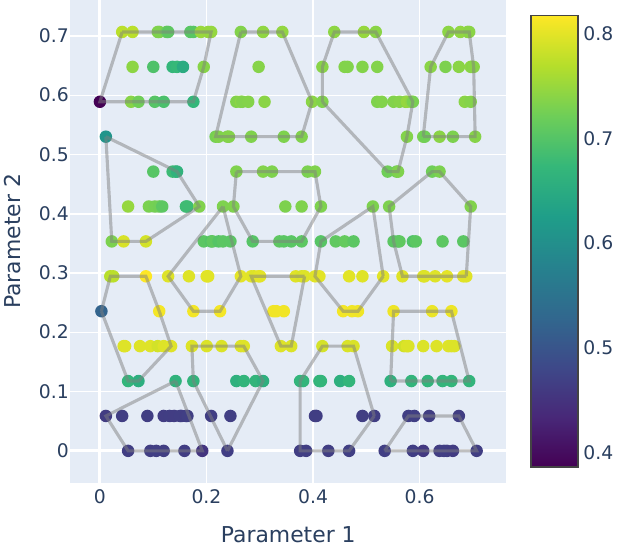}}
    \caption{\revision{Results from a real storage system. 
    (a) The cumulative regret of \EXP32 for reward from a real system. % as a function of steps. 
    (b) The regret as a function of the number of clusters.
    Star-shaped points are the regret results of runs on the real system.
    (c) Distribution of the estimation of the Lipschitz constant for the empirical reward function vs. a shuffled one.
    (d) Empirical mean of each arm on the storage system. The dashed line polygon represents the partition over arms used in \EXP32.}
    }
    \label{fig:ExpReplay}
%  \end{minipage}
  \Description{Four-panel real-storage results: regret vs.\ steps, regret vs.\ clusters, empirical-Lipschitz histogram, and empirical mean-reward heatmap with cluster boundaries.}
\end{figure*}

% \begin{figure}
%     \centering
%     %height
%     \subfigure[Empirical Lipschitz]
%     {\label{fig:lipshitz}
%     \includegraphics[width=.48\linewidth]{figures/snap_lipfshitz.pdf}}
%     \subfigure[Empirical mean]
%     {\label{fig:snap_reward}
%     \includegraphics[width=.48\linewidth]{figures/snap_clustered_arms.pdf}}
%     \caption{\revision{Results from a real storage system. 
%         (a) Distribution of the estimation of the Lipschitz constant for the empirical reward function vs. a shuffled one.
%     (b) Empirical mean of each arm on the storage system. The dashed line polygon represents the partition over arms used in \EXP32.}
%     }
%     \label{fig:ExpReplayDetails}
%     \Description{}
% \end{figure}

\subsection{Clustering with Lipschitz: Much to Gain}\label{sec:Clustering_with_Lipschitz}

In this subsection, we consider the case that the rewards of the arms form a metric space and, in particular, the simpler case where, within each cluster, we have a \emph{Lipschitz} condition ensuring that the rewards of all arms in a cluster are ``close'' to each other. For this simpler case, we do not even require a condition for the distance between clusters.
Such a situation fits, for example, in a system's parameters optimization problem where each arm corresponds to a set of parameters trying to optimize a certain objective (e.g., power, delay, etc.) \revision{(as we saw in Section \ref{sec:experimental_results})}. 
% The set of potential arms is extensive since we can tune each parameter to many possible values. In turn, we assume a smoothness condition on each parameter, such that each minor change of a parameter results in a minor shift in the outcome. The ``state'' of the system (e.g., number of active jobs) is unknown and can change over time (e.g., as new jobs come and go), even dramatically, but for every state, the Lipschitz condition holds (Eq. \eqref{eq:Lipschitz}). %\revisionKDD{COMMET (Roey) The example feels a bit redundant, especially since this is the technical section with proofs. Maybe we can remove it if we have to trim something out}
%
More formally, for the clustering scenario, we define an $\ell$-partition as,
\begin{definition}[$\ell$-partition]
An arms partition $\cP$ is an $\ell$-partition if the following Lipschitz condition holds:
\begin{align*}
    \forall P \in \cP, a,b \in P, ~~~ \vert c_t(a)-c_t(b) \vert \le \mathcal{D}(a,b) \le \ell, ~~\forall \text{ rounds } t.
\end{align*}    
\end{definition}

%\chen{Fix}
%It is important to note that $\ell$ is related to the range of the possible rewards, not only their expected value (if such exists). For example, for a Bernoulli random variable $\ell=1$.
%\revision{We note that in the above Lipschitz condition, we considered that the adversarial rewards could be random. In this case $G_{\max}$ is also defined as the expected value of the best arm. }
Following Theorem \ref{thm:clusters}, and since within each cluster, the \revision{expected} regret is at most $\ell$, we can easily state the following.

\begin{corollary}\label{cor:weakLip}
    For $k$  nonstochastic arms, $T>0$ and an $\ell$-partition $\cP$ of the arms, the regret of the \EXP32 algorithm using EXP3 is bounded as follows:
    \begin{align*}
        &G_{\max} - \exp[G_{\EXP32}] \le  
        \cO\left(\sqrt{p T \log p }\right) + \cO\left(\sqrt{k^* T \log k^*}\right) + T \cdot \ell
    \end{align*}
\end{corollary}

% \begin{align}
%     G_{\max} - \sum_{t=1}^T r_t(a_t) 
%     &\le G_{\max} - \sum_{t=1}^T (r^*_t(C_t) - \ell) \\
%     &= G_{\max} - \sum_{t=1}^T r^*_t(C_t) + T \ell \\
%     &=\cO(\sqrt{c T \log c}) + T \ell
% \end{align}
% To achieve $\Tilde{O}(\sqrt{cT})$ we need
% \begin{align}
% \ell \le \sqrt{\frac{c}{T}}
% \end{align}

Note that the above bound does not assume any contribution to the execution of EXP3 within clusters, resulting from the $\ell$-partition. Using an appropriate A-MAB algorithm within each Lipschitz cluster potentially allows us to improve the result.

Let ALB$^+$ denote an adaptive Lipschitz bandits algorithm satisfying the following property: 
\begin{property}\label{prp:alb}
For the Adversarial Lipschitz MAB setting with $k'$ arms satisfying the Lipschitz condition (Eq. \ref{eq:Lipschitz}) and a maximum distance $\ell$, the regret is 
bounded by $\cO\left(\ell \cdot \sqrt{k' T \log k'} \right)$.    
\end{property}

\revisionKDD{Property \ref{prp:alb} is a strong property, as it depends on the reward structure. However, algorithms satisfying this property exist, for example by assuming side-information. A concrete example is the EXP3-SET algorithm \cite{alon2017nonstochastic} which utilizes a feedback graph $G_t$, (as the side-information), which is directed and time-changing. If $G_t=G$ is symmetric, and $\alpha(G)$ is the independence number of $G$, EXP3-SET achieves a regret of $\cO(\sqrt{\alpha(G)T\log k'})$.}

\revisionKDD{To satisfy Property \ref{prp:alb}, we construct a feedback graph $G$ where an edge exists between two arms in $G$ iff for two arms $a, a'$, $\cD(a, a') \le \epsilon$ for an appropriate selection of $\epsilon$. Under the Lipschitz condition, observing the reward of one arm reveals information about its neighbors in $G$. Consequently, if the independence number satisfies $\alpha(G) \le \ell^2 k'$, EXP3-SET fulfills Property \ref{prp:alb}.
}

\revisionKDD{Note that if we assume a single cluster and a flat ALB$^+$ algorithm, then $\ell$ might be large (e.g., 1). But when we break the arms into clusters, we may have an $\ell$-partition, where within each cluster $\ell$ might be significantly smaller. The following theorem formalized the potential gains. 
}

\begin{theorem}\label{thm:ClustersLipschitz}
    For $k$  nonstochastic arms, $T>0$ and an $\ell$-partition $\cP$ of the arms, the regret of the \EXP32, using an ALB$^+$ algorithm can be bounded as follows:
    \begin{align}
        G_{\max} - \exp[G_{\EXP32}] &\le 
        \cO\left(\sqrt{p T \log p }\right)
         +\cO\left(\ell \sqrt{k^* T \log k^*}\right) \notag \\
         &\;\;\;+\cO\left(\ell \cdot \sqrt{k T \log (k/p) }\right), \label{eq:theorem}
    \end{align}
where $p$ is the number of clusters and $k^*$ is the size of the cluster with the best arm in hindsight.
\end{theorem}

\revision{
\begin{proof}[Proof of Theorem \ref{thm:ClustersLipschitz}]
    The proof follows from Theorem \ref{thm:clusters}, the proof of Eq. \ref{eq:R2}, and Property \ref{prp:alb}.
    We can improve Eq. \ref{eq:R2} by plugging Property \ref{prp:alb} so for each cluster $i$ its internal regrate is bounded by $\cO\left(\ell \cdot \sqrt{\card{\cP^i} \card{T^i} \log \card{\cP^i}} \right)$. 
    In turn, the worst case for the sum $\sum_{i=1}^{p }\cO\left(\ell \sqrt{\card{P^i} \card{T^i} \log \card{P^i}}\right)$ is still the case where $\forall i, \card{\cP^i}= \frac{k}{p}$ and $\card{T^i} = \frac{T}{p}$ and the results follows.
    %and the result in \cite{auer2002nonstochastic} (Corollary 3.2)
    %that if for each arm the reward is in the range $[a, b]$, $a < b \le 1$ the regret can be bounded (using translation and rescaling) as $\cO\left(\ell \cdot \sqrt{k T \log k}\right)$, where $\ell=b-a$. Using this result, we can extend it to the case of multiple clusters (and different times), each running its own EXP3, if they all obey the $\ell$ Lipschitz condition, as required by the theorem.
\end{proof}
}

%\chen{UP TO HERE. Need to talk about EXP3 when it holds. fix Takeaway and example}

Following Theorem \ref{thm:ClustersLipschitz} we can state  the second main takeaway of the paper. 

\paragraph{{\small $\blacksquare$} Takeaway.} \textit{The results hold for nonstochastic arms, which keeps the Lipschitz condition. 
Comparing the upper bound of the ``flat'' (single cluster) ALB$^+$ algorithm, we can achieve improvement when the number of clusters and the size of the ``best'' cluster are relatively small. i.e.,  $p \ll k $, and $k^* \ll k $, and when the \revisionKDD{$\ell$-partition is sufficiently tight,}
%Lipschitz constant is sufficiently small. 
i.e., $\ell \le \frac{1}{\sqrt{k}}$.}

As a concrete example, we can state the following

\begin{corollary}
  Given an $\ell$-partition $\cP$, with $\ell \le \frac{1}{\sqrt[4]{k}}$ and $\sqrt{k}$ clusters, each of size $\sqrt{k}$, the \EXP32, using an ALB$^+$ algorithm  can be bounded as $\cO(\sqrt{k^{1/2} T \log k})$, which is $\Omega(\sqrt[4]{k})$ time better than the flat ALB$^+$ algorithm.
\end{corollary}

\subsection{\revisionCR{Discussion: Theory and Practice}}

\revisionCR{We close this section by relating the assumptions and parameters of our theoretical analysis to the empirical setting, and by discussing how each affects \EXP32 in practice.}

\revisionCR{
\para{Property~\ref{prp:alb}.}
Although this property is required for our improved theoretical bound, our empirical evaluation deliberately uses Tsallis-INF, which does not formally satisfy Property~\ref{prp:alb} and, unlike EXP3-SET, does not assume any side-information on the relations between the arms. We chose Tsallis-INF as it is the state-of-the-art ``best-of-both-worlds'' algorithm, offering superior stability and performance in practice. As we showed in the previous sections, clustering the arms and using EXP3-based or Tsallis-INF algorithms for parent and child levels still significantly improves performance, demonstrating that the hierarchical \EXP32 framework is robust even with off-the-shelf algorithms that do not formally satisfy this property.
}

\revisionCR{
\para{Practical implications of $\ell$, $p$, and $k^*$.}
The parameters $\ell$, $p$, and $k^*$ in our analysis carry significant practical weight. Specifically, $\ell$ represents the reward smoothness within a cluster; in real-world settings, proximate metric bandits often exhibit similar performance, yielding a small $\ell$ that allows our hierarchical approach to exploit this structure. This is empirically validated in our production storage system (Figure~\ref{fig:lipshitz}), which exhibits a low $\ell$ value. The parameters $p$ and $k^*$ are governed by the clustering configuration, highlighting a theoretical trade-off between a large number of clusters $p$ (and thus a smaller $k^*$) or vice-versa. This trade-off is corroborated by the U-shape observed in our empirical experiments (e.g., Figure~\ref{fig:2D_non_clusters}). Compared to the flat approach, our framework incurs a computational overhead that may be non-negligible for small numbers of arms (small $k$), as discussed in our runtime analysis (Figure~\ref{fig:run-times}).
}

% If we run EXP3 within each cluster (still assuming Lipschitz condition within each cluster), we can improve the result as follows.
% For a Lipschitz cluster with $k_i$ arms, the regret bound is (see \cite{auer2002nonstochastic} after Corolary 3.2) $O(\ell \sqrt{k_i T \log k_i}$. We have, 
% \begin{align}
%     G_{\max} - \sum_{t=1}^T r_t(a_t) 
%     &=\cO(\sqrt{c T \log c}) + \sum_{i=1}^{c}\cO(\ell \sqrt{\frac{k}{c}T_i \log \frac{k}{c}})
% \end{align}
% The worst case is for $T_i=\frac{T}{c}$. Then we have 
% \begin{align}
%     G_{\max} - \sum_{t=1}^T r_t(a_t) 
%     &=\cO(\sqrt{c T \log c}) + \sum_{i=1}^{c}\cO(\ell \sqrt{\frac{k}{c^2} T \log \frac{k} {c}}) \\
%     &=\cO(\sqrt{c T \log c}) +\cO(\ell \sqrt{k T \log \frac{k}{c}})
% \end{align}
% The minimum is achieved when 
% \begin{align}
%     \sqrt{c T \log c} \approx  \ell \sqrt{k T \log \frac{k}{c}}
% \end{align}
% Which gives 
% \begin{align}
%     \ell \approx  \sqrt{\frac{c \log c}{k \log \frac{k}{c}}}
% \end{align}
% If the above condition holds, then the regret is $O(\sqrt{c T \log c})$ (compare to $O(\sqrt{k T \log k}))$.

% To find the best clustering we need to solve (try)
% \begin{align}
%     \min_c ~~\text{s.t.}~~ \ell \le \sqrt{\frac{c \log c}{k \log \frac{k}{c}}} \approx \sqrt{\frac{c}{k}} << \sqrt{\frac{c}{T}}.
% \end{align}

\section{Conclusion}
%\vspace{-3pt}
\revisionKDD{This paper introduced a practical hierarchical MAB framework to address the challenge of optimizing decisions}
%This paper introduced a novel nonstochastic, metric-based MAB framework to address the challenge of optimizing decisions
in dynamic environments with large, 
structured action spaces, such as automated system configuration. We proposed \EXP32, a hierarchical algorithm that leverages the clustered, Lipschitz nature of the action space. 
Our theoretical analysis demonstrated that \EXP32 achieves robust worst-case performance, matching traditional methods, 
while offering significant improvements when the underlying structure is favorable, as demonstrated by an improved regret bound under Lipschitz conditions. Importantly, these theoretical findings were validated through both simulations and experiments on a real storage system, confirming \EXP32's ability to achieve lower regret and faster convergence in practice. Future work includes exploring adaptive clustering techniques, multilevel hierarchies of clusters, and studying distributed settings.
%the setting where multiple arms can be sampled in parallel.

%\balance
\bibliographystyle{ACM-Reference-Format}
\bibliography{references}

\appendix

\section{Technical Appendix}

\subsection{Computing Infrastructure Used for Running Experiments}

All experiments were run on a standard MacBook Pro laptop with 
Apple M2 Pro Chip, 12 Cores, and 32 GB Memory.

%\subsection{More Details for Figures Within the Paper}\label{apn:detailed}
%Figure \ref{fig:4stochastic} extends Figure \ref{fig:2stochastic}, showing the behavior of the EXP and EXP++ algorithms.
% Figure \ref{fig:4staticArmNonstochastic} extends Figure \ref{fig:2staticArmNonstochastic}, showing the behavior of the EXP and UCB algorithms.

%Figure \ref{fig:simNonstDetails} provides an example of the "traveling arm" scenario where the best arm changes in time.

%Figure \ref{fig:ExpReplayDetails} provides additional details on the dataset and the results from a real storage system.

% \begin{figure*}
%     \centering
%     \includegraphics[width=\linewidth]{figures/regret_comparison_same_delta_env_stochastic_4.pdf}
%     \caption{Comparing \EXP32 (with $\sqrt{k} = 16$ clusters) with ``flat'' algorithms (single cluster) for different well-known MAB algorithms. Stochastic Scenario from \cite{zimmert2021tsallis}, $k=256$.}
%     \label{fig:4stochastic}
%     \Description{}
% \end{figure*}

% \begin{figure*}
%     \centering
%     \includegraphics[width=\linewidth]{figures/regret_comparison_same_delta_env_adversarial_4.pdf}
%     \caption{Fixed Optimal Arm, and Nonstochastic Scenario from \cite{zimmert2021tsallis}, $k=256$. Comparing \EXP32 (with $\sqrt{k} = 16$ clusters) with ``flat'' algorithms (single cluster) for different well-known MAB algorithms.}
%     \label{fig:4staticArmNonstochastic}
%     \Description{}
% \end{figure*}

\subsection{1D Stochastic, Metric Example}
\label{sec:ablation}

Figure \ref{fig:sim1D} presents the results of the one-dimensional setup, where the $k$ arms are equally spaced on the range $[0,1]$.
Figure \ref{fig:1D_time} shows the regret as a function of time and compares the flat case ($p=1$) to \EXP32 with $p=4$ and $p=16$ clusters.  
Figure \ref{fig:1D_clusters} presents the regret at $T=10k, 50k$, and $100k$ for different numbers of clusters.
Both subfigures demonstrate that \EXP32 has significantly lower regret for the best clustering, $p=16$, relative to the flat baseline: \roey{$1618\pm$30} vs. \roey{$4866\pm54$}, respectively (t-test p-value=\roey{$9\times10^{-25}$}), an improvement of \roey{66.8}\%.  Moreover, it shows that \EXP32 is not much worse (only for $p=2$, the results were a bit lower due to the symmetry of the problem for $a=\frac{1}{2}$). 

Another important property is the U shape appearing in Figure \ref{fig:1D_clusters}. This is a result of Eq. \eqref{eq:theorem} from Theorem \ref{thm:ClustersLipschitz} where the first term is monotone increasing with $p$, while the second term is monotone decreasing with $p$, leading to the U shape when the third term is not dominating.

%\subsection{result for 2D static}

\subsection{2D Stochastic, Metric Example}
In order to verify the extendability to a higher dimension, we repeated the experiment for two dimensions $d=2$, i.e., we are trying to optimize two parameters. \revisionCR{Specifically, $k=256$ arms are placed on a $16\times 16$ grid in $[0,1]^2$, and the $\sqrt{k}=16$ clusters used by \EXP32 correspond to the natural $4\times 4$ partition into contiguous square tiles, each containing $16$ arms.} Figure \ref{fig:sim2Dsetup} depicts the arms' location in the metric space where each marker corresponds to a different cluster, and the color map indicates the expected reward. 
The Lipschitz condition is held by the setup we use (Section \ref{sec:empirical}).
Figure \ref{fig:sim2D} again depicts the benefit of \EXP32 over using a flat Tsallis-INF, Figure \ref{fig:2D_time} as a function of time and Figure \ref{fig:2D_clusters} as a function of the number of clusters. For \EXP32 the regret at $T=100$k was $\roey{1675\pm213}$ and for the flat Tsallis-INF $\roey{4859\pm72}$, respectively (t-test p-value=$\roey{7.5\times10^{-14}}$), an improvement of about \roey{66.5}\%.

\subsection{Effect of Number of Arms}
The previous two sections clearly demonstrated the benefit of \EXP32 over the flat Tsallis-INF. In order to study the effect of the number of arms $k$ on that benefit we iterated over $k=2^j$ where $j\in\revisionKDD{\{4,...,12\}}$. Figure \ref{fig:effect_of_arms} measures the ratio of the \EXP32 regret relative to the flat one and shows that the benefit increases as the number of arms increases. This is indeed expected from the results Section \ref{sec:Clustering_with_Lipschitz}.  We note also that, as expected, a larger dimension will require more arms to show the benefit.

\begin{figure}
    \centering
    %height
    \subfigure[Regret vs. Steps]
    {\label{fig:1D_time}
    \includegraphics[width=.48\linewidth]{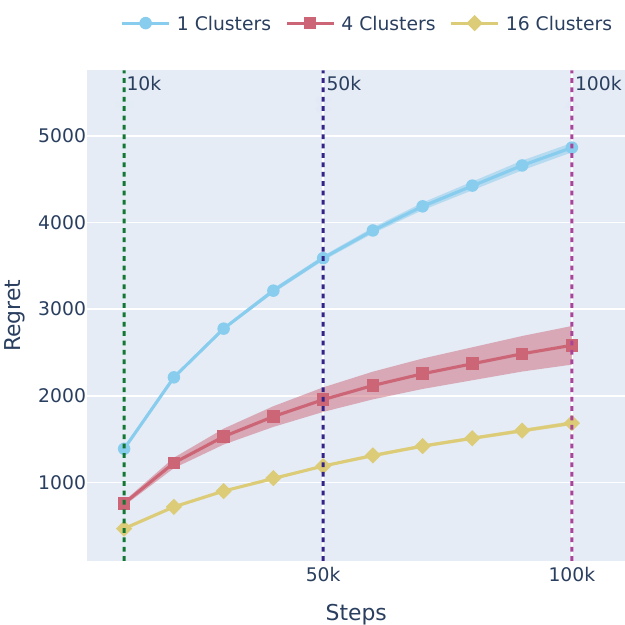}}
    \subfigure[Regret vs. Clusters]
    {\label{fig:1D_clusters}
    \includegraphics[width=.48\linewidth]{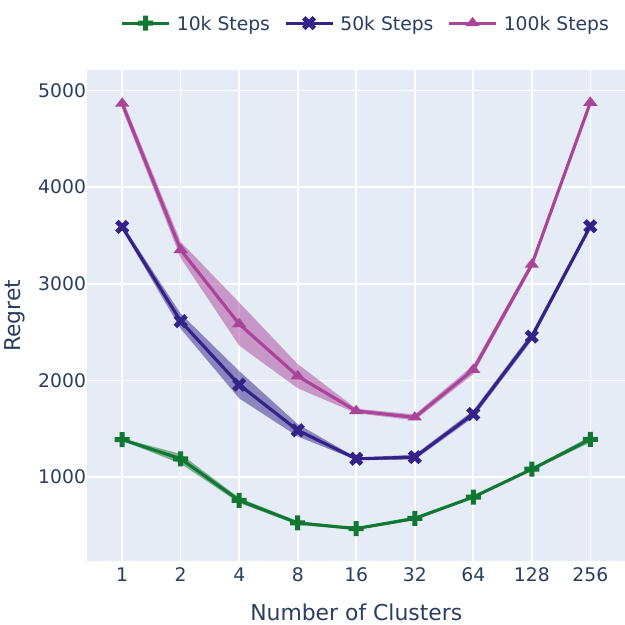}}
    \caption{Cumulative regret for \EXP32 vs. flat Tsallis-INF ($k=1$ and $k=256$) for the 1D setup. (a) \roey{As a function of steps. (b) As a function of the number of clusters.}}
    % \label{fig:sim_info}
    \label{fig:sim1D}
    \Description{Two-panel plot of cumulative regret vs.\ steps and vs.\ number of clusters for the 1-D setup.}
\end{figure}

\begin{figure}
    \centering
    \subfigure[Regret vs. Steps]
    {\label{fig:2D_non_time_random}
    \includegraphics[width=.45\linewidth]{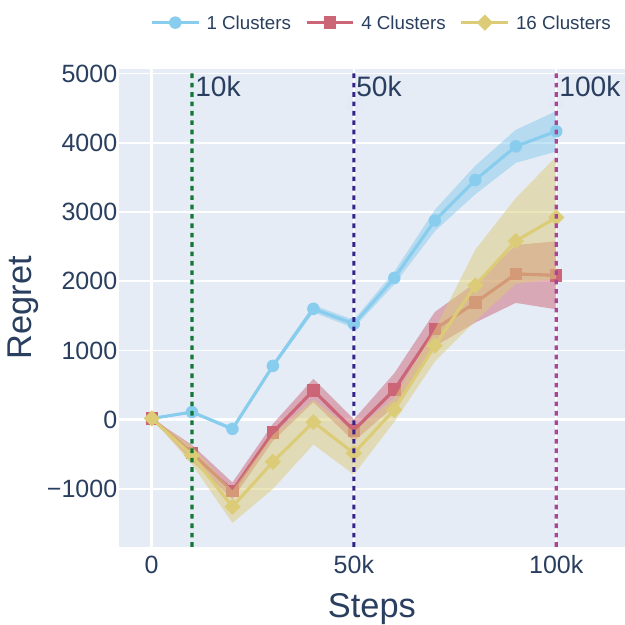}}
    \subfigure[Regret vs. Clusters]
    {\label{fig:2D_non_clusters_random}
    \includegraphics[width=.45\linewidth]{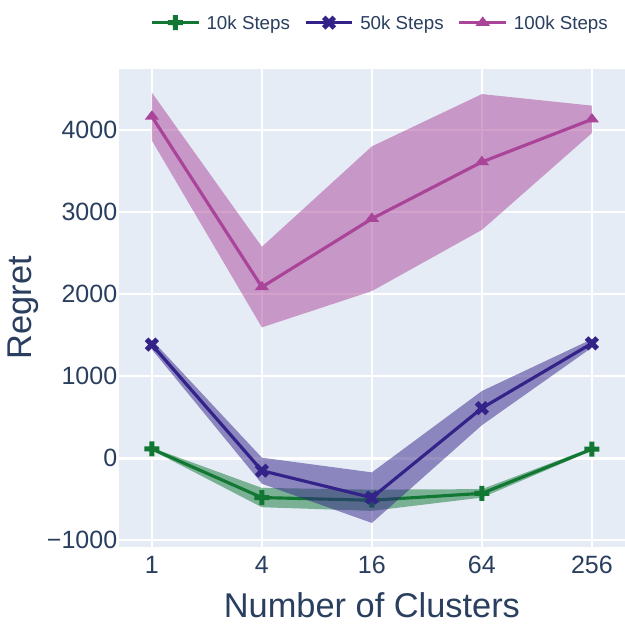}}
        \caption{\revision{The ``Traveling'' Optimal Arm, Nonstochastic (adversarial) and Metric Space Scenario with {\bf random clustering}. (a) The cumulative regret for ABoB as a function of steps. (b) The cumulative regret for ABoB as a function of the number of clusters.}}
    \label{fig:simNonstResultsRandomClustering}
    \Description{Two-panel plot: regret vs.\ steps and vs.\ number of clusters for the traveling-arm scenario with random clustering.}
\end{figure}

\subsection{Effect of Random Clustering}

Our setup for the nonstochastic (adversarial) and metric spaces scenario (``traveling arm'') in Section \ref{sec:empirical} considers clustering of the arms by the Lipschitz condition of the 2D metric. Our theoretical results indicate that for any clustering we cannot lose too much.
Figure \ref{fig:simNonstResultsRandomClustering} repeats the same experiment of Figure \ref{fig:simNonstResults} but with random partition of the arms. As we can see, also in this case the clustering improves the regret, but to a smaller extent.

\begin{figure}[H]
    \centering
    \subfigure[Regret vs. Steps]{
    \includegraphics[width=0.46\linewidth]{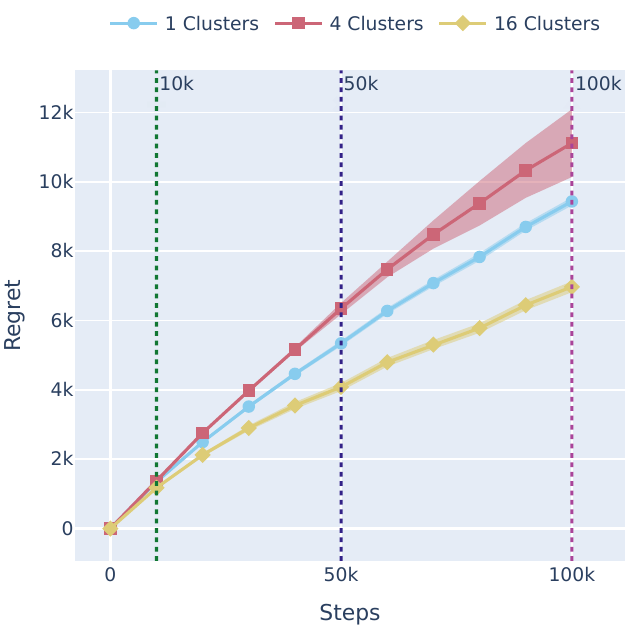}
    }
    \subfigure[Regret vs. Clusters]{
    \includegraphics[width=0.46\linewidth]{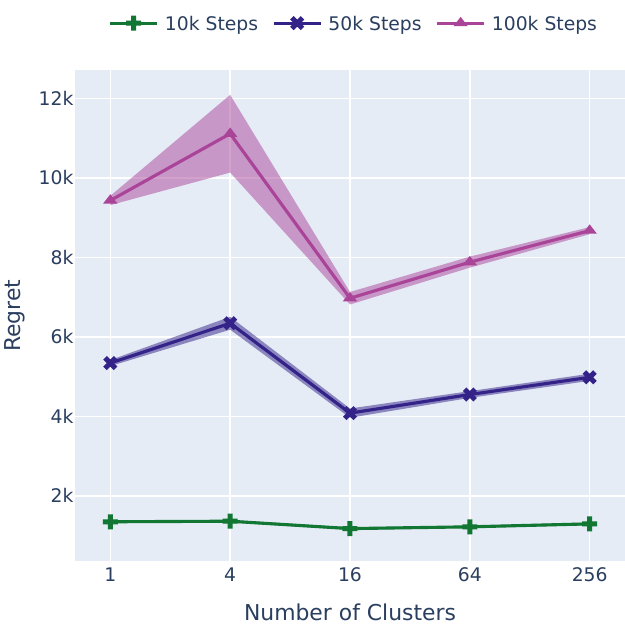}
    }
    \caption{Same as Figure \ref{fig:ExpReplay} but using EXP3}
    \label{fig:ExpReplayEXP3}
    \Description{Two-panel plot: regret vs.\ steps and vs.\ number of clusters for the real-storage replay using EXP3.}
\end{figure}

The best result obtained by using 4 clusters, where \EXP32 achieves a regret of $2086\pm491$ compared to the flat Tsallis-INF baseline which achieves a regret of $4166\pm291$, (t-test p-value $10^{-8}$). This is about 50\% improvement, but less than the 91\% improvement achieved by the Lipschitz-based clustering.
\revisionCR{This empirically validates the ``Not Much to Lose'' guarantee of Theorem~\ref{thm:clusters}: even when clusters are constructed without any geometric information about the arms (i.e., the Lipschitz structure is broken at the partition level), \EXP32 still matches or beats the flat baseline, because the parent A-MAB can identify the cluster containing the optimal arm. The remaining gap to the Lipschitz-aware partition (Figure~\ref{fig:simNonstResults}) is precisely the ``Much to Gain'' term in Theorem~\ref{thm:ClustersLipschitz}, governed by the in-cluster Lipschitz constant $\ell$.}

\begin{figure*}
    \centering
    %height
    \subfigure[2D setup]{\label{fig:sim2Dsetup}
    \includegraphics[width=0.26\linewidth]{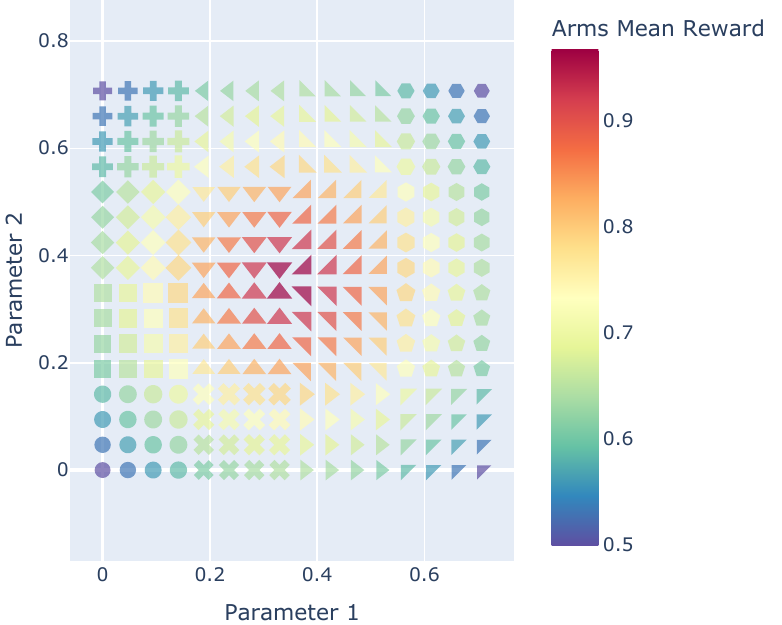}}
    \subfigure[Regret vs. Steps]
    {\label{fig:2D_time}
    \includegraphics[width=.23\linewidth]{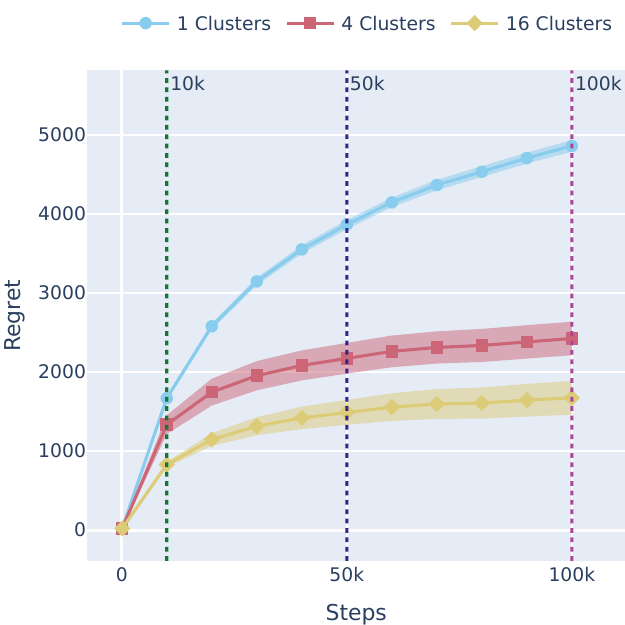}}
    \subfigure[Regret vs. Clusters]
    {\label{fig:2D_clusters}
    \includegraphics[width=.23\linewidth]{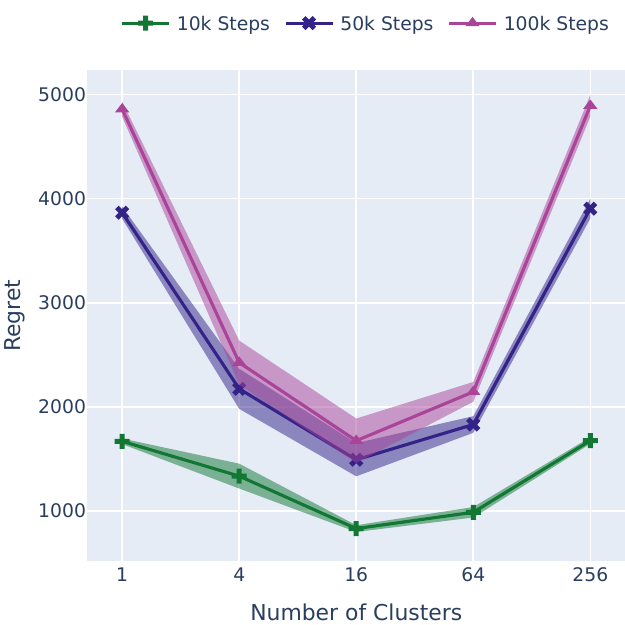}}
    \subfigure[Numbers of arms]{\label{fig:effect_of_arms}
    \includegraphics[width=0.22\linewidth]{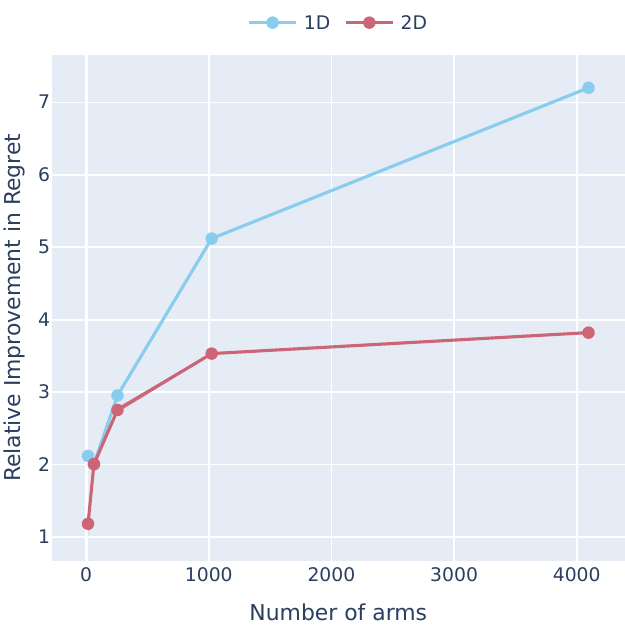}}
    \caption{
    %A nonstochastic setup, where $a(t)$ changes over time (left, line), causing arms reward to change (right heatmap). This in turn causes the best arm to change (right, line), and the location of the best arm to change as well (left, heatmap represents the fraction the arm was the best).
    Cumulative regret for \EXP32 vs. flat Tsallis-INF ($k=1$ and $k=256$) for the 2D setup. 
    (a) A 2D stochastic setup where each marker is an arm with a different reward, and arms are partitioned into clusters. Rewards and the partition are according to Section \ref{sec:empirical}. 
    (b) As a function of steps. 
    (c) As a function of the number of clusters.
    (d) Effect of the number of arms on the relative regret of \EXP32 to flat Tsallis-INF for the 1D and 2D cases. More arms imply a larger ratio.
    }
    \label{fig:sim2D}
    \Description{Four-panel plot: 2-D arm-space scatter with cluster boundaries, regret vs.\ steps, regret vs.\ number of clusters, and relative regret as a function of the number of arms.}
\end{figure*}

%\roey{
\subsection{Bandit Algorithm}
In our results, we have considered the use of the Tsallis-INF Multi-Armed Bandit algorithm in our hierarchical approach. \revisionCR{To verify that the hierarchical advantage is not specific to Tsallis-INF, we re-ran the real-system replay using EXP3~\cite{auer2002nonstochastic} as both the parent and the child bandit. Figure~\ref{fig:ExpReplayEXP3} shows the regret achieved by flat EXP3 versus hierarchical EXP3 (\EXP32). The hierarchical variant continues to dominate the flat one across the full range of cluster counts, even though EXP3 is, in absolute terms, a weaker learner than Tsallis-INF in this setting. This indicates that the gain from clustering is largely orthogonal to the choice of base bandit: any reasonable adversarial MAB that satisfies the standard $\tilde O(\sqrt{kT})$ regret will inherit the \EXP32 speed-up under the same Lipschitz assumption, with absolute performance scaling with the quality of the base learner.}
%}

\subsection{Stochastic Setup on a Real System}
In our experimental results (Section ~\ref{sec:experimental_results}), we have considered a scenario where the system experiences dynamically changing workloads, which puts us in the non-stochastic setting. Here we show that under a static workload (i.e. stochastic setting) we achieve better performance in regret. Under this scenario, Figure \ref{fig:ExpReplayStatic} shows that the hierarchical using 16 clusters achieves a regret of $3708\pm45$ (compared to $5349\pm46$ in the dynamic setting) which is dominant over the flat approach achieving regret of $5587\pm26$ (compared to $7568\pm43$ in the dynamic setting) (t-test p-value $1.19\times 10^{-22}$).
%\revisionCR{This setting is closer to the classical stochastic MAB regime under which the Lipschitz-clustering bound (Theorem~\ref{thm:ClustersLipschitz}) was originally derived, so the larger relative gain over the flat baseline is consistent with our analysis: when the reward distribution is stationary, the in-cluster Lipschitz constant $\ell$ is fully exploited by the child bandit.}

\subsection{Validation: In-cluster Similarity}
\revisionKDD{In this section we examine our setup of Section \ref{sec:experimental_results}.
We ask, could it be the case that our clustering was such that in each cluster all arms are very "close", so "close" that even a simple uniform sampling bandit at the child level will result in good performance? 
This is clearly not the case for the stochastic and nonstochastic scenarios from \cite{zimmert2021tsallis}. Both are cases of a ``needle in a haystack'', so clearly a uniform sampling as the child bandit will fail. But what about the scenario of the traveling arm in the 2-D grid? Figure \ref{fig:Tsalis-uniform} presents the results of such an experiment. It compares the run of Tsallis-Tsallis \EXP32 with a Tsallis-Uniform \EXP32, and shows that each cluster has arms with significant variation. This is true since o.w. the results of the Tsallis-Uniform \EXP32 would not deviate as much from the Tsallis-Tsallis \EXP32.}

\begin{figure}[H]
    \centering
    \includegraphics[width=.6\linewidth]{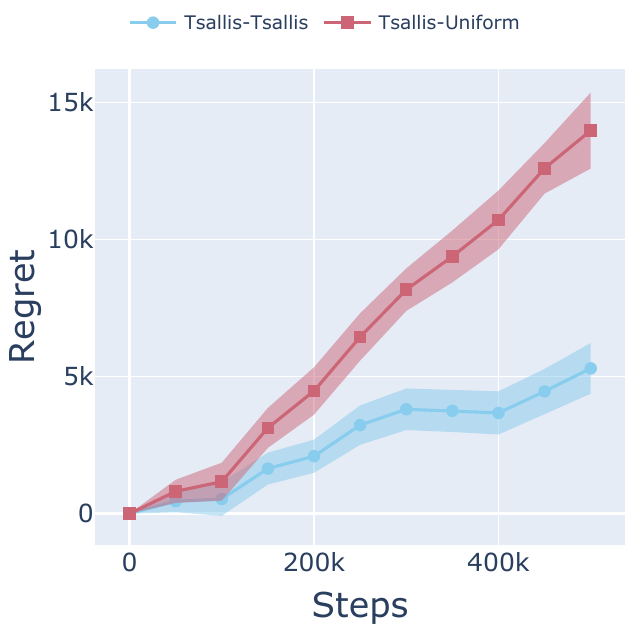}
    \caption{Traveling-arm 2-D grid scenario: Tsallis-Tsallis \EXP32 vs.\ Tsallis-Uniform \EXP32. The gap rules out the hypothesis that clusters are internally uniform enough for a uniform child bandit to suffice, confirming significant intra-cluster reward variation.}
    \label{fig:Tsalis-uniform}
    \Description{Cumulative-regret curves comparing Tsallis--Tsallis ABoB with Tsallis--Uniform ABoB on the 2-D traveling-arm grid.}
\end{figure}

\begin{figure}[H]
    \centering
    %height
    \subfigure[Regret vs. Steps]
    {\label{fig:2D_replay_time_stochastic}
    \includegraphics[width=.48\linewidth]{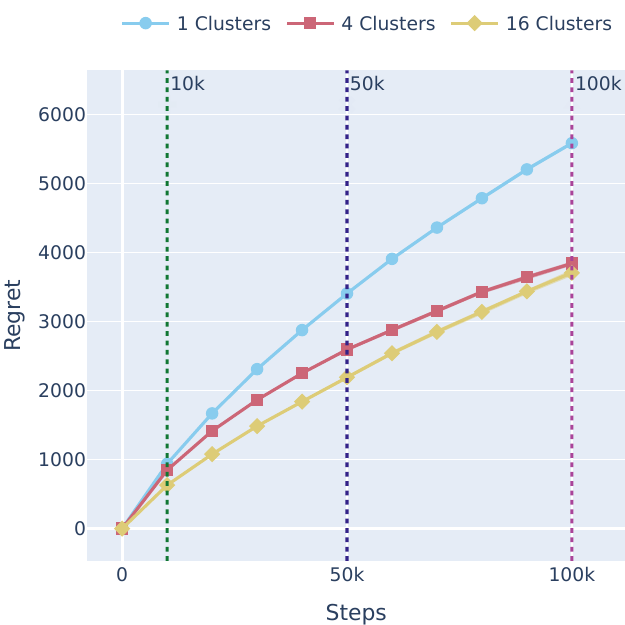}}
    \subfigure[Regret vs. Clusters]
    {\label{fig:2D_replay_clusters_stochastic}
    \includegraphics[width=.48\linewidth]{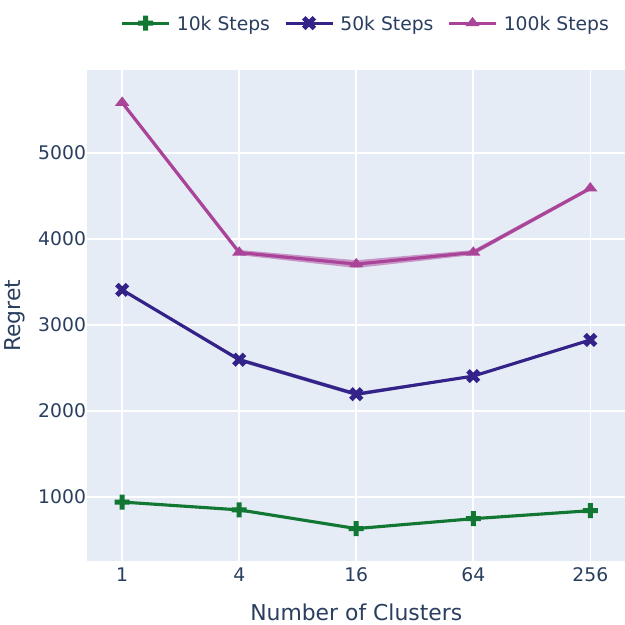}}
    \caption{Same as Figure \ref{fig:ExpReplay} but the system experiences a static workload.}
    \label{fig:ExpReplayStatic}
    \Description{Two-panel plot: regret vs.\ steps and vs.\ number of clusters for the real-storage replay under a static workload.}
\end{figure}

\subsection{Different Parent-Child Algorithms}

\revisionCR{To assess the sensitivity of \EXP32 to the choice of base bandit at each level of the hierarchy, we ran every parent$\times$child pairing over four representative MAB algorithms: EXP3, EXP3++, Tsallis-INF, and UCB. The resulting two grids appear in Figure~\ref{fig:stochasticParent-Child} (stochastic environment) %, as in Figure~\ref{fig:4stochastic})
 and Figure~\ref{fig:adversarialParent-Child} (nonstochastic environment).
 %, as in Figure~\ref{fig:4staticArmNonstochastic}).
Each row corresponds to a fixed parent and each column to a fixed child algorithm, with $p=1$ (flat, child only), $p=\sqrt{k}=16$ clusters (\EXP32), and $p=k=256$ (parent only).} Note that for the same parent and child algorithm (diagonal), the case of 1 and 256 clusters is the same. 
For all other cases, in each row, the case of 256 clusters remains the same (only child algorithm works), and for each column, the case of a single cluster remains the same (only parent algorithm works).

\begin{figure}[H]
    \centering
    \includegraphics[width=.92\linewidth]{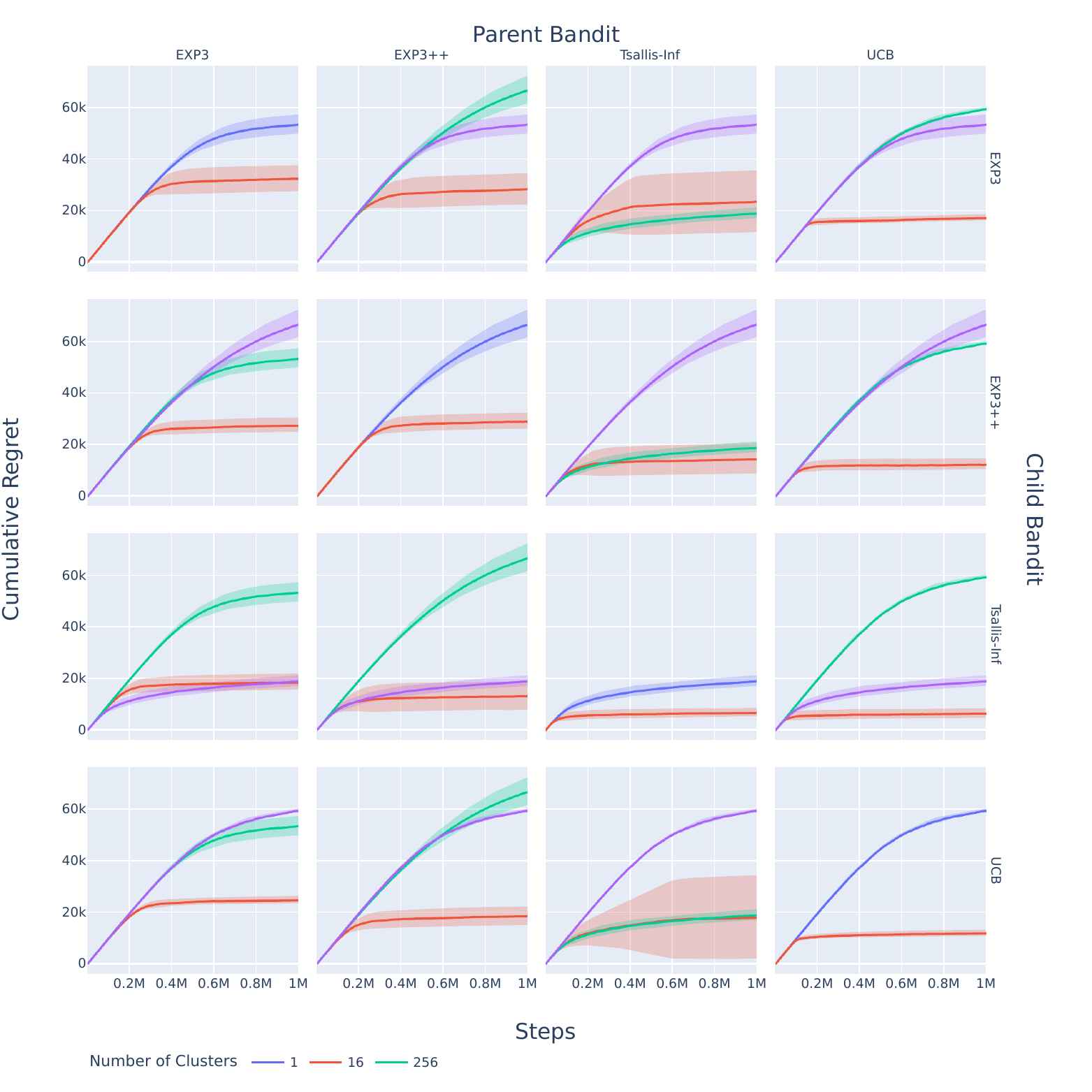}
    \caption{Comparing different parent-child combinations in \EXP32.
    Stochastic environment as in Figure \ref{fig:4stochastic}.}
    \label{fig:stochasticParent-Child}
    \Description{Grid of regret curves for different parent--child algorithm combinations under a stochastic environment.}
\end{figure}

\begin{figure}[H]
    \centering
    \includegraphics[width=.92\linewidth]{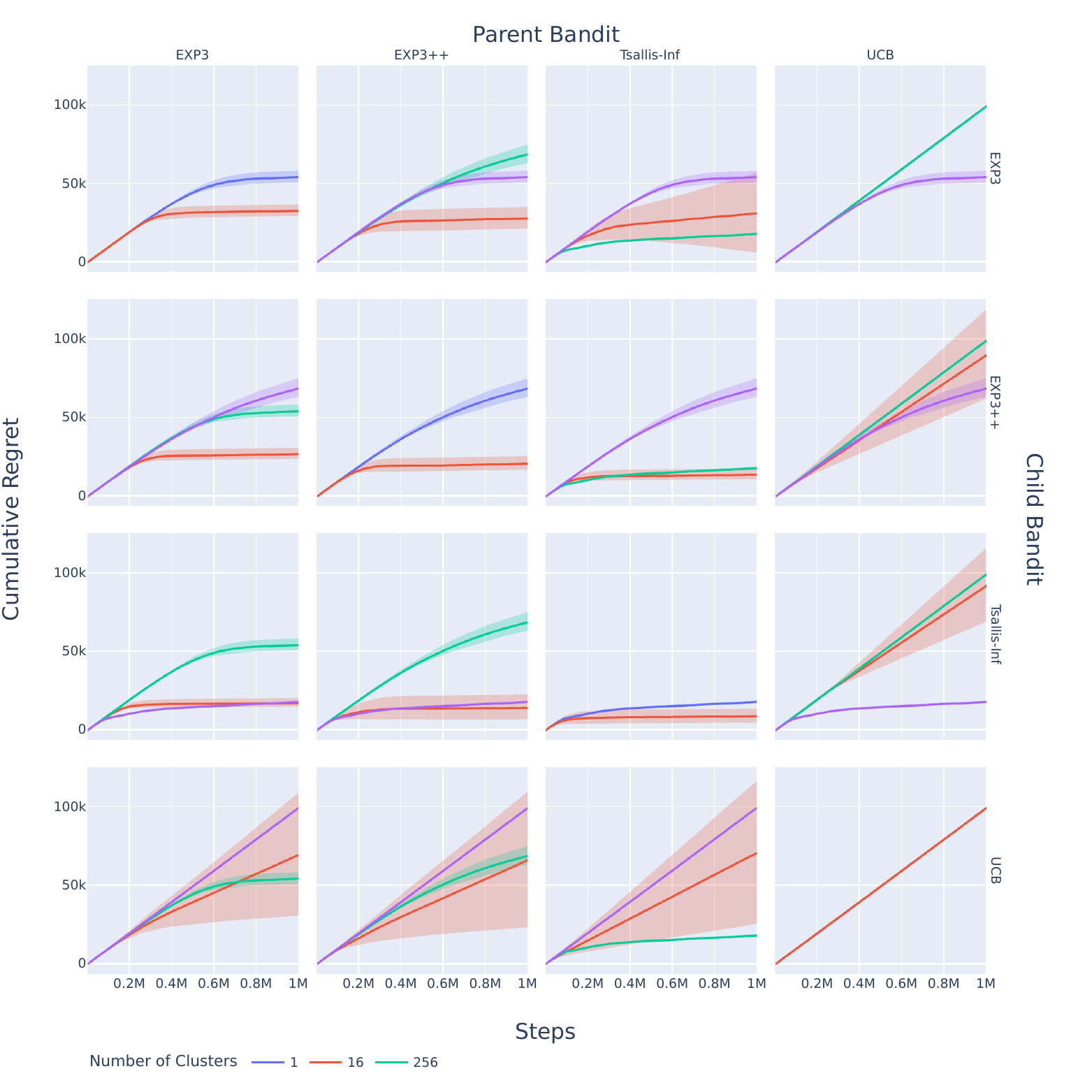}
    \caption{Comparing different parent-child combinations in \EXP32.
    Nonstochastic environment as in Figure \ref{fig:4staticArmNonstochastic}.}
    \label{fig:adversarialParent-Child}
    \Description{Grid of regret curves for different parent--child algorithm combinations under a nonstochastic environment.}
\end{figure}

\revisionCR{Three observations stand out across both environments. First, \EXP32 ($p=\sqrt{k}$) consistently dominates or matches both extremes for every choice of base bandits, confirming that the hierarchical regret reduction is not an artifact of a single algorithm pair. Second, the choice of \emph{child} bandit influences \EXP32's performance more than the choice of \emph{parent} bandit, since the child sees the residual variability inside a cluster, where the bulk of the per-step decisions are made. Third, Tsallis-INF as the child consistently produces the lowest regret, motivating our default Tsallis-Tsallis configuration; pairings such as EXP3$\to$Tsallis-INF or UCB$\to$Tsallis-INF remain competitive when the parent's exploration profile must be controlled separately.}

\end{document}